\documentclass[twoside,11pt]{article}

\usepackage{bbm}
\usepackage{color}

\usepackage{jmlr2e}
\usepackage{epsfig}
\usepackage{amsmath, amsfonts, amssymb}
\usepackage{natbib}
\usepackage{graphicx}

\usepackage{hyperref}
\usepackage{algorithm,subfig,float,times}
\usepackage[algo2e]{algorithm2e}

\newcommand{\be}{\begin{eqnarray}}
\newcommand{\ee}{\end{eqnarray}}
\newcommand{\E}{\mathbb{E}}
\newcommand{\n}{\nonumber \\}

\newcommand{\btheta}{\mathbf{\theta}}
\newcommand{\bx}{\mathbf{x}}
\newcommand{\g}{\, | \,}
\newcommand{\parhead}[1]{\textit{#1 \,}}
\newcommand{\lambdamax}{\lambda_{\mathrm{max}}}
\newcommand{\lambdamin}{\lambda_{\mathrm{min}}}
\newcommand{\errmax}{\mathrm{err_\mathrm{max}}\newcommand{\Bt}{B_{\epsilon/S} }}
\renewcommand{\v}{\textrm{v} }
\newcommand{\half}{\textstyle{\frac{1}{2}}}
\newcommand{\invS}{{\textstyle \frac{1}{\sqrt{S}}}}

\jmlrheading{18}{2017}{1-35}{4/17; Revised 10/17}{12/17}{17-214}{Stephan Mandt, Matthew D. Hoffman, David M. Blei}
\ShortHeadings{Stochastic Gradient Descent as Approximate Bayesian Inference}{Mandt, Hoffman, and Blei}
\firstpageno{1}

\newtheorem{assumption}{Assumption}

\begin{document}

\title{\textbf{Stochastic Gradient Descent as Approximate Bayesian Inference}}

\author{\name Stephan Mandt \email stephan.mandt@gmail.com \\
\addr Data Science Institute\\Department of Computer Science\\ Columbia University\\New York, NY 10025, USA
\AND
\name Matthew D. Hoffman \email mathoffm@adobe.com \\
\addr Adobe Research\\ Adobe Systems Incorporated\\
601 Townsend Street\\
San Francisco, CA 94103, USA
\AND
\name David M. Blei \email david.blei@columbia.edu \\
\addr Department of Statistics\\ Department of Computer Science\\Columbia University\\New York, NY 10025, USA
}
\editor{Manfred Opper}

\maketitle

\begin{abstract}%
  Stochastic Gradient Descent with a constant learning rate (constant SGD) simulates a Markov chain with a stationary distribution. 
  With this perspective, we derive several new results. (1) We show that constant SGD can be used
  as an approximate Bayesian posterior inference algorithm. Specifically, we show how to adjust the tuning
  parameters of constant SGD to best match the stationary distribution to a posterior, minimizing the Kullback-Leibler divergence between these two distributions.
  (2) We demonstrate that constant SGD gives rise to a new variational
  EM algorithm that optimizes hyperparameters in complex probabilistic models.
  (3) We also show how to tune SGD with momentum for approximate sampling.
  (4) We analyze stochastic-gradient MCMC algorithms. For Stochastic-Gradient Langevin Dynamics and Stochastic-Gradient Fisher Scoring, we quantify the approximation errors due to finite learning rates.
  Finally (5), we use the stochastic process perspective to give a short proof of why Polyak averaging is optimal. Based on this idea, we propose a scalable approximate MCMC algorithm, the Averaged Stochastic Gradient Sampler.
\end{abstract}
\begin{keywords}⟨approximate Bayesian inference, variational inference, stochastic optimization, stochastic gradient MCMC, stochastic differential equations⟩\end{keywords}
\section{Introduction}

Stochastic gradient descent (SGD) has become crucial to modern machine
learning. SGD optimizes a function by following noisy gradients with a
decreasing step size. The classical result
of~\citet{robbins1951stochastic} is that this procedure provably
reaches the optimum of the function (or local optimum, when it is
nonconvex)~\citep{bouleau1994numerical}. Recent studies investigate the merits of adaptive step
sizes~\citep{duchi2011adaptive, tieleman2012lecture}, gradient or 
iterate averaging~\citep{toulistowards,defossez2015averaged}, and constant
step-sizes~\citep{bach2013non, flammarion2015averaging}. Stochastic 
gradient descent has enabled efficient optimization with massive data,
since one can often obtain noisy-but-unbiased gradients very cheaply by
randomly subsampling a large dataset.

Recently, stochastic gradients (SG) have also been used in the service of
scalable Bayesian Markov Chain Monte Carlo (MCMC) methods, where the goal is
to generate samples from a
conditional distribution of latent variables given a data set.  In
Bayesian inference, we
assume a probabilistic model $p(\btheta, \bx)$ with data $\bx$ and
hidden variables $\btheta$; our goal is to approximate the posterior
\begin{align}
  \label{eq:intro-posterior}
  p(\btheta \g \bx) = \exp\{\log p(\btheta, \bx) - \log p(\bx)\}.
\end{align}
These so-called stochastic gradient MCMC algorithms---such as SG Langevin
dynamics~\citep{welling2011bayesian}, SG Hamiltonian
Monte Carlo~\citep{chen2014stochastic}, SG thermostats~\citep{ding2014bayesian},
and SG Fisher
scoring~\citep{ahn2012bayesian}---employ stochastic gradients of
$\log p(\btheta, \bx)$ to improve convergence and computation of
existing sampling algorithms.  Also see~\citet{ma2015complete} for a
complete classification of these algorithms.

The similarities between SGD as an optimization algorithm and
stochastic gradient MCMC algorithms raise the
question of how exactly these two types of algorithms relate to each
other. The main questions we try to address in this paper are:
\begin{itemize}
  \item What is the simplest modification to SGD that
    yields an efficient approximate Bayesian sampling algorithm?
  \item How can we construct other sampling algorithms based
on variations of SGD such as preconditioning~\citep{duchi2011adaptive,tieleman2012lecture}, momentum~\citep{polyak1964some} or Polyak averaging~\citep{polyak1992acceleration}?
\end{itemize}
To answer these questions, we draw on
the theoretical analysis tools of continuous-time stochastic
differential equations~\citep{bachelier1900theorie, gardiner1985handbook} and variational inference~\citep{Jordan:1999}.

As a simple example, consider SGD with a
constant learning rate (constant SGD). Constant SGD first marches toward
an optimum of the objective function and then bounces around its vicinity. (In contrast,
traditional SGD converges to the optimum by decreasing the
learning rate.)  Our analysis below rests on the idea that constant SGD is a stochastic process with a stationary distribution,
one that is centered on the optimum and that has a certain
covariance structure. The main idea is that we can use this
stationary distribution to approximate a posterior. In contrast, 
stochastic gradient MCMC algorithms take precautions to sample
from an asymptotically exact posterior, but at the expense of slower
mixing. Our inexact approach enjoys minimal implementation effort and typically
faster mixing. It is a hybrid of Monte Carlo and variational
algorithms, and complements the toolbox of approximate Bayesian inference.

Here is how it works. We apply constant SGD as though we were trying
to minimize the negative log-joint probability $-\log p(\btheta, \bx)$
over the model parameters $\btheta$. Constant SGD has several tunable
parameters: the constant learning rate, the minibatch size, and the
preconditioning matrix (if any) that we apply to the gradient
updates. These tuning parameters all affect the stationary
distribution of constant SGD; depending on how they are set, this
stationary distribution of $\btheta$ will be closer to or farther from
the posterior distribution $p(\btheta | \bx)$. If we set these
parameters appropriately, we can perform approximate Bayesian inference
by simply running constant SGD.

In more detail, we make the following contributions:
\begin{itemize} 

\item[1.] First, we develop a variational Bayesian view of stochastic
  gradient descent. Based on its interpretation as a
  continuous-time stochastic process---specifically a multivariate
  Ornstein-Uhlenbeck (OU) process~\citep{uhlenbeck1930theory,gardiner1985handbook}---we compute stationary
  distributions for a large class of SGD algorithms, all of which
  converge to a Gaussian distribution with a non-trivial covariance
  matrix. The stationary distribution is parameterized by the learning
  rate, minibatch size, and preconditioning matrix.

  Results about the multivariate OU process make it easy to minimize the
  KL divergence between the stationary distribution and the posterior analytically.
  We can thus relate the optimal step size or preconditioning matrix to the Hessian and noise covariances near the optimum.
  The optimal preconditioners relate to
  AdaGrad~\citep{duchi2011adaptive},
  RMSProp~\citep{tieleman2012lecture}, and classical Fisher
  scoring~\citep{longford1987fast}.  
  We demonstrate how these different
  optimization methods compare when used for approximate inference.

\item[2.] We show that constant SGD gives rise to a new variational EM algorithm~\citep{Bishop:2006} which
  allows us to use SGD to optimize hyperparameters while
  performing approximate inference in a Bayesian model.
  We demonstrate this by fitting a posterior to a Bayesian multinomial regression model.
  
\item[3.] We use our formalism to derive the stationary distribution for SGD with
  momentum~\citep{polyak1964some}. Our results show that adding momentum only changes the
  scale of the covariance of the stationary distribution, not its shape.
  This scaling factor is a simple function of the
  damping coefficient. Thus we can also use SGD with momentum for approximate Bayesian inference.

\item[4.] Then, we analyze scalable MCMC
  algorithms. Specifically, we use the stochastic-process perspective
  to compute the stationary distribution of Stochastic-Gradient Langevin Dynamics (SGLD) by~\citet{welling2011bayesian} when using
  constant learning rates, and analyze stochastic
  gradient Fisher scoring (SGFS) by~\citet{ahn2012bayesian}.   
  The view from the multivariate OU process reveals a simple
  justification for this method: we confirm that the preconditioning
  matrix suggested in SGFS is indeed optimal. We also derive a
  criterion for the free noise parameter in SGFS that can enhance numerical stability, and we show how the
  stationary distribution is modified when the preconditioner is
  approximated with a diagonal matrix (as is often done in practice
  for high-dimensional problems).

\item[5.] Finally, we analyze iterate averaging~\citep{polyak1992acceleration}, where one successively averages
	the iterates of SGD to obtain a lower-variance estimator of the optimum. Based on the stochastic-process methodology,
	we give a shorter derivation of a known result, namely
        that the convergence speed of iterate averaging cannot be
        improved by preconditioning the stochastic gradient with any
        matrix. Furthermore, we show that (under certain assumptions),
        Polyak iterate averaging can yield an \emph{optimal}
        stochastic-gradient MCMC algorithm, and that this optimal
        sampler can generate exactly one effectively independent
        sample per pass through the dataset. This result is both
        positive and negative; it suggests that iterate averaging can
        be used as a powerful Bayesian sampler, but it also argues
        that no SG-MCMC algorithm can generate more than one useful
        sample per pass through the data, and so the cost of these
        algorithms must scale linearly with dataset size.

\end{itemize}

Our paper is organized as follows. In Section~\ref{sec:formalism} we
review the continuous-time limit of SGD, showing that it can be
interpreted as an OU process. In Section~\ref{sec:consequences} we
present consequences of this perspective: the interpretation of SGD as
variational Bayes and results around preconditioning and momentum.
Section~\ref{sec:MCMC} discusses SG Langevin Dynamics and SG Fisher Scoring.
In Section~\ref{sec:iterate} we discuss Polyak averaging for optimization and sampling.
In the empirical study (Section \ref{sec:experiments}), we show that our theoretical
assumptions are satisfied for different models, that we can use
SGD to perform gradient-based hyperparameter optimization, and that iterate averaging gives rise to a Bayesian sampler with fast mixing.

\section{Related Work}
Our paper relates to Bayesian inference and stochastic optimization.

\parhead{Scalable MCMC.} Recent work in Bayesian statistics focuses on
making MCMC sampling algorithms scalable by using stochastic
gradients. In particular, \citet{welling2011bayesian} developed
stochastic-gradient Langevin dynamics (SGLD). This algorithm samples
from a Bayesian posterior by adding artificial noise to the stochastic
gradient which, as the step size decays, comes to dominate the SGD noise. Also see
\citet{sato2014approximation} for a detailed convergence analysis of
the algorithm. Though elegant, one disadvantage of SGLD is that the
step size must be decreased to arrive at the correct sampling
regime, and as step sizes get small so does mixing speed. Other
research suggests improvements to this issue, using Hamiltonian
Monte Carlo~\citep{chen2014stochastic} or
thermostats~\citep{ding2014bayesian}. 
 \citet{shang2015covariance} build on thermostats and use a similar continuous-time
formalism as used in this paper.

\citet{ma2015complete} give a
complete classification of possible stochastic gradient-based MCMC
schemes.

Below, we will analyze properties of stochastic gradient Fisher
scoring~\citep[SGFS;][]{ahn2012bayesian}, an extention to SGLD. This
algorithm speeds up mixing times in SGLD by preconditioning gradients
with the inverse gradient noise covariance. \citet{ahn2012bayesian}
show that (under some assumptions) SGFS can eliminate the bias
associated with running SGLD with non-vanishing learning rates.  Our
approach to analysis can extend and sharpen the results of
\citet{ahn2012bayesian}. For example, they propose using a diagonal
approximation to the gradient noise covariance as a heuristic; in our
framework, we can analyze and rigorously justify this choice as a
variational Bayesian approximation.

\citet{maclaurin2015early} also interpret SGD as a non-parametric
variational inference scheme, but with different goals and in a
different formalism. The paper proposes a way to track entropy changes
in the implicit variational objective, based on estimates of the
Hessian. As such, the authors mainly consider sampling distributions that are
not stationary, whereas we focus on constant learning rates and
distributions that have (approximately) converged. Note that their
notion of hyperparameters does not refer to model parameters but to parameters of SGD.

\parhead{Stochastic Optimization.} Stochastic gradient descent is an
active field~\citep{zhang2004solving,bottou1998online}. Many papers
discuss constant step-size SGD. \citet{bach2013non,
  flammarion2015averaging} discuss convergence rate of averaged
gradients with constant step size, while~\citet{defossez2015averaged}
analyze sampling distributions using quasi-martingale techniques.
\citet{toulis2014statistical} calculate the asymptotic variance of SGD
for the case of decreasing learning rates, assuming that the data is
distributed according to the model. None of these papers consider the
Bayesian setting. \citet{dieuleveut2017bridging} also analyzed SGD with
constant step size and its relation to Markov chains. Their analysis resulted in a novel extrapolation
scheme to improve the convergence behavior of iterate averaging.

The fact that optimal preconditioning (using a decreasing
Robbins-Monro schedule) is achieved by choosing the inverse noise
covariance was first shown in~\citep{sakrison1965efficient}, but here
we derive the same result based on different arguments and
suggest a scalar prefactor. Note the optimal scalar learning rate of
$2/{\rm Tr}(BB^\top)$, where $BB^\top$ is the SGD noise covariance (as discussed in Section~\ref{sec:consequences} or this paper), can also be derived based on stability
arguments. This was done in the context of least mean square
filters~\citep{widrow1985adaptive}.

Finally,~\citet{chen2015bridging} also draw analogies between SGD and
scalable MCMC. They suggest annealing the posterior over time to
use scalable MCMC as a tool for global optimization. We follow the
opposite idea and suggest to use constant SGD as an approximate
sampler by choosing appropriate learning rate and preconditioners.

\parhead{Stochastic differential equations.} The idea of analyzing
stochastic gradient descent with stochastic differential equations is
well established in the stochastic approximation
literature~\citep{kushner2003stochastic, ljung2012stochastic}. Recent
work focuses on dynamical aspects of the
algorithm. \citet{li2015dynamics} discuss several one-dimensional
cases and momentum. \citet{li2017stochastic} give a mathematically
rigorous justification of the continuous-time limit.~\citet{chen2015convergence} analyze
stochastic gradient MCMC and study their convergence properties
using stochastic differential equations.

Our work makes use of the same formalism but has a different focus.
Instead of analyzing dynamical properties, we focus on stationary
distributions. Further, our paper introduces the idea of minimizing KL
divergence between multivariate sampling distributions and the
posterior.

\parhead{Variational Inference.} 
Variational Inference (VI) denotes a set of methods which aim at approximating
a Bayesian posterior by a simpler, typically factorized distribution. This is done
by minimizing Kullback-Leibler divergence or related divergences 
between these distributions~\citep{Jordan:1999,opper2001advanced}.
For the class of models where the conditional distributions are
in the exponential family, the variational objective can be optimized by
closed-form updates~\citep{ghahramani2001propagation}, but this is a restricted class
of models with conjugate priors. 
A scalable version of VI, termed Stochastic Variational Inference (SVI), relies on stochastic gradient descent
for data subsampling~\citep{hoffman2013stochastic}.
For non-conjugate models, Black-Box variational inference~\citep{ranganath2014black}
has enabled SVI for a large class of models, but this approach may suffer from 
high-variance gradients.
A modified form of black-box variational inference
relies on re-parameterization gradients~\citep{salimans2013fixed,kingma2013auto,
  rezende2014stochastic,kucukelbir2015automatic,ruiz2016generalized}. This version
is limited to continuous latent variables but typically has much lower-variance gradients.

In this paper, we compare against the Gaussian reparameterization gradient version
of black-box variational inference as used in~\citet{kingma2013auto,
  rezende2014stochastic,kucukelbir2015automatic} which we refer to as BBVI.
We find that our approach performs similarly in practice, but it is
different in that it does not optimize the parameters of a simple variational distribution.
Rather, it controls the shape of the approximate posterior via the
parameters of the optimization algorithm, such as the learning rate or preconditioning matrix.

\section{Continuous-Time Limit Revisited}
\label{sec:formalism}

We first review the theoretical framework that we use throughout the
paper.  Our goal is to characterize the behavior of SGD when using a
constant step size.  To do this, we approximate SGD with a
continuous-time stochastic
process~\citep{kushner2003stochastic,ljung2012stochastic}.

\subsection{Problem Setup}

Consider loss functions of the following form:
\begin{align}
\label{eq:loss-one}
{\cal L}(\theta)
= {\textstyle\frac{1}{N} \sum_{n=1}^N} \ell_n(\theta), \quad
g(\theta)\equiv \nabla_\theta \cal{L}(\theta).
\end{align}
Such loss functions are common in machine learning, where
${\cal L}(\theta) \equiv {\cal L}(\theta,\bx)$ is a loss function that
depends on data $\bx$ and parameters $\theta$. Each
$ \ell_n(\theta)\equiv \ell(\theta,\bx_n)$ is the contribution to the
overall loss from a single observation $\bx_n$.  For example, when
finding a maximum-a-posteriori estimate of a model, the contributions to the
loss may be
\begin{align}
\label{eq:loss-two}
\ell_n(\theta) = - \log p(x_n \,|\,
\theta) - \textstyle{\frac{1}{N}} \log p(\theta),
\end{align}
where $p(x_n \,|\, \theta)$ is the likelihood and $p(\theta)$ is the prior.
For simpler notation, we will suppress the dependence of the loss on
the data.

From this loss we construct stochastic gradients.  Let ${\cal S}$ be a
set of $S$ random indices drawn uniformly at random from the set
$\{1,\ldots, N\}$.  This set indexes functions $\ell_n(\theta)$, and
we call $\cal{S}$ a ``minibatch'' of size $S$.  Based on the
minibatch, we used the indexed functions to form a stochastic estimate
of the loss and a stochastic gradient, 
\begin{align}
\label{eq:SG}
\hat{\cal L}_S(\theta) = \textstyle\frac{1}{S} \sum_{n \in {\cal S}}\, \ell_n(\theta), \quad \hat{g}_S(\theta) = \nabla_\theta\hat{\cal L}_S(\theta).
\end{align}
In expectation the stochastic gradient is the full gradient, i.e., $g(\theta) = {\mathbb E}[\hat{g}_S(\theta)]$.
We use this stochastic gradient in the SGD update
\begin{align}
\label{eq:SGD}
\theta(t+1)  =  \theta(t) - \epsilon \, \hat{g}_S(\theta(t)).
\end{align}
Above and for what follows, we assume a constant (non-decreasing) learning rate $\epsilon$.

Eqs.~\ref{eq:SG} and~\ref{eq:SGD} define the discrete-time
process that SGD simulates from. We will approximate it with a
continuous-time process that is easier to analyze.
  
\subsection{SGD as an Ornstein-Uhlenbeck Process}

We now show how to approximate the discrete-time Eq.~\ref{eq:SGD} with
a continuous-time Ornstein-Uhlenbeck
process~\citep{uhlenbeck1930theory}.  This leads to the stochastic
  differential equation below in Eq.~\ref{eq:OU-process}.  To justify
  the approximation, we make four assumptions.  We verify its accuracy
  empirically in Section~\ref{sec:experiments}.

\begin{assumption}
 Observe that the stochastic gradient is a sum
of $S$ independent, uniformly sampled contributions.  Invoking the
central limit theorem, we assume that the gradient noise is Gaussian
with covariance $\frac{1}{S} C(\theta)$, hence
\begin{align}
\label{eq:SGD2}
\hat{g}_S(\theta)  \approx  g(\theta) + {\textstyle \frac{1}{\sqrt{S}}}\Delta g(\theta),\quad  \Delta g(\theta) \sim {\cal N}(0,C(\theta)). 
\end{align}
\end{assumption}


\begin{assumption}
We assume that the covariance matrix $C(\theta)$ is
approximately constant with respect to $\theta$. As a symmetric positive-semidefinite matrix,
this constant matrix $C$ factorizes as
\begin{align}
\label{eq:CBB}
C(\theta) \approx C  = B B^\top.
\end{align}
\end{assumption}
Assumption 2 is justified when the iterates of SGD are confined to
a small enough region around a local optimum of the loss
(e.g. due to a small $\epsilon$) that the noise
covariance does not vary significantly in that region.

We now define $\Delta \theta(t) = \theta (t+1) - \theta(t)$ and combine
Eqs.~\ref{eq:SGD},~\ref{eq:SGD2}, and~\ref{eq:CBB} to rewrite the
process as
\begin{align}
\label{eq:differenceeq}
\Delta \theta(t)  =  -\epsilon \, g(\theta(t))
+\textstyle \frac{\epsilon}{\sqrt{S}}B \,\Delta W, \quad \Delta W \sim
{\cal N}\left(0,{\bf I}\right).
\end{align}
This can be interpreted as a finite-difference equation that approximates
the following continuous-time stochastic differential equation:
\begin{align}
\label{eq:Langevin}
d \theta(t)  =  - \epsilon g(\theta) dt +\textstyle \frac{\epsilon}{\sqrt{S}}B  \,  dW(t).
\end{align}

\begin{assumption}
We assume that we can approximate the finite-difference 
equation (\ref{eq:differenceeq}) by the stochastic differential
equation (\ref{eq:Langevin}). 
\end{assumption}
This assumption is justified if
either the gradients or the learning rates are small enough
that the discretization error becomes negligible. 

\begin{assumption}
We assume that the stationary
distribution of the iterates is constrained to a region where the loss
is well approximated by a quadratic function,
\begin{align}
{\cal L}(\theta)  =  \textstyle\frac{1}{2} \, \theta^\top A \theta.
\end{align}
(Without loss of generality, we assume that a minimum of the loss is at $\theta=0$.)
We also assume that $A$ is positive definite.
\end{assumption}
The symmetric matrix $A$ is thus the Hessian at the optimum.
Assumption 4 makes sense when the loss function is smooth and
the stochastic process
reaches a low-variance quasi-stationary distribution around a deep local minimum.
The exit time of a stochastic process is typically
exponential in the height of the barriers between minima, which can make local optima very stable even in the presence of noise~\citep{kramers1940brownian}.

\parhead{SGD as an Ornstein-Uhlenbeck process.} The
four assumptions above result in a specific kind of stochastic process, the
multivariate \emph{Ornstein-Uhlenbeck}
process~\citep{uhlenbeck1930theory}:
\begin{align}
\label{eq:OU-process}
d \theta(t)  =  -\epsilon A\, \theta(t) dt + \invS \epsilon B\, dW(t) 
\end{align}

This connection helps us analyze properties of SGD because the
Ornstein-Uhlenbeck process has an analytic stationary distribution
$q(\theta)$ that is Gaussian. This distribution will be the core analytic tool
of this paper: 
\begin{align}
\label{eq:stationarydistribution}
q(\theta)  \propto 
\exp\left\{-\textstyle\frac{1}{2}\theta^\top \Sigma^{-1}\theta\right\}. 
\end{align}
The covariance $\Sigma$ satisfies
\begin{align}
\label{eq:stationaryvariance}
\Sigma A + A \Sigma = \textstyle\frac{\epsilon}{S}BB^\top.
\end{align}
(More details are in Appendix~\ref{sec:covariance}.) 
Without explicitly solving this equation, we see that the resulting
covariance $\Sigma$ is proportional to the learning rate $\epsilon$ and
inversely proportional to
the magnitude of $A$
and minibatch size
$S$.   This characterizes the
stationary distribution of running SGD with a constant step size.

\parhead{Discussion of Assumptions 1--4.}
Our analysis suggests that constant SGD and Langevin-type diffusion algorithms~\citep{welling2011bayesian}
are very similar.  Both types of algorithms can be characterized by three regimes.
First, there is a search phase where the algorithm approaches the optimum. In this early phase, assumptions 1--4
are often violated and it is hard to say anything general about the behavior of the
algorithm. Second, there is a phase where SGD has converged to the vicinity of a local optimum.
Here, the objective already looks quadratic, but the gradient noise is small relative to the average gradient $g(\theta)$.
Thus SGD takes a relatively directed path towards the optimum. This is the regime where our assumptions
should be approximately valid, and where our formalism reveals its use.
Finally, in the third phase the iterates are near the local optimum. 
Here, the average gradient $g(\theta)$ is small and the sampling noise becomes more important. In this final phase,
constant SGD begins to sample from its stationary distribution.

Finally, we note that if the gradient noise covariance $C$ does
not have full rank, then neither will the stationary covariance
$\Sigma$. This scenario complicates the analysis in
Section~\ref{sec:consequences}, so below we will assume that $C$
has full rank. This could be enforced easily by adding very
low-magnitude isotropic artificial Gaussian noise to the stochastic
gradients.

\section{SGD as Approximate Inference}
\label{sec:consequences}

We discussed a continuous-time interpretation of SGD with a constant
step size (constant SGD).  We now discuss how to use constant SGD as
an approximate inference algorithm.  To repeat the set-up from the
introduction, consider a probabilistic model $p(\btheta, \bx)$ with
data $\bx$ and hidden variables $\btheta$; our goal is to approximate
the posterior $p(\btheta \,|\, \bx)$ in Eq.~\ref{eq:intro-posterior}.

We set the loss to be the negative log-joint distribution
(Eqs.~\ref{eq:loss-one} and \ref{eq:loss-two}), which equals the
negative log-posterior up to an additive constant.  The classical goal of SGD is to
minimize this loss, leading us to a maximum-a-posteriori point estimate of the
parameters. This is how SGD is used in many statistical models,
including logistic regression, linear regression, matrix
factorization, and  neural networks.  In
contrast, our goal here is to tune the parameters of SGD so that
its stationary distribution approximates the posterior.

\begin{figure}
  \centering
  \includegraphics[width=.32\linewidth]{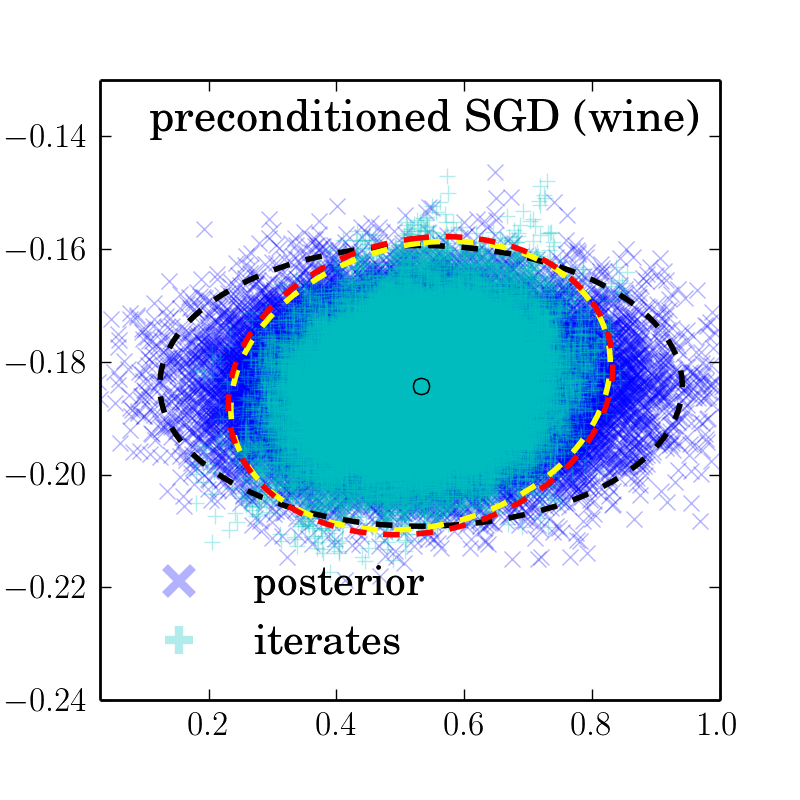}
  \includegraphics[width=.32\linewidth]{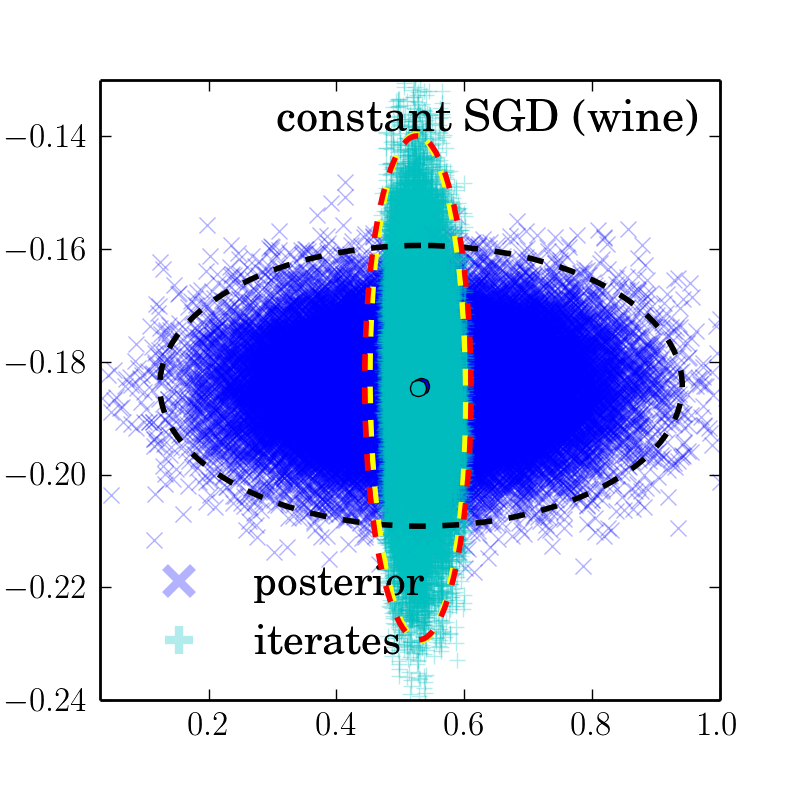}
  \includegraphics[width=.32\linewidth]{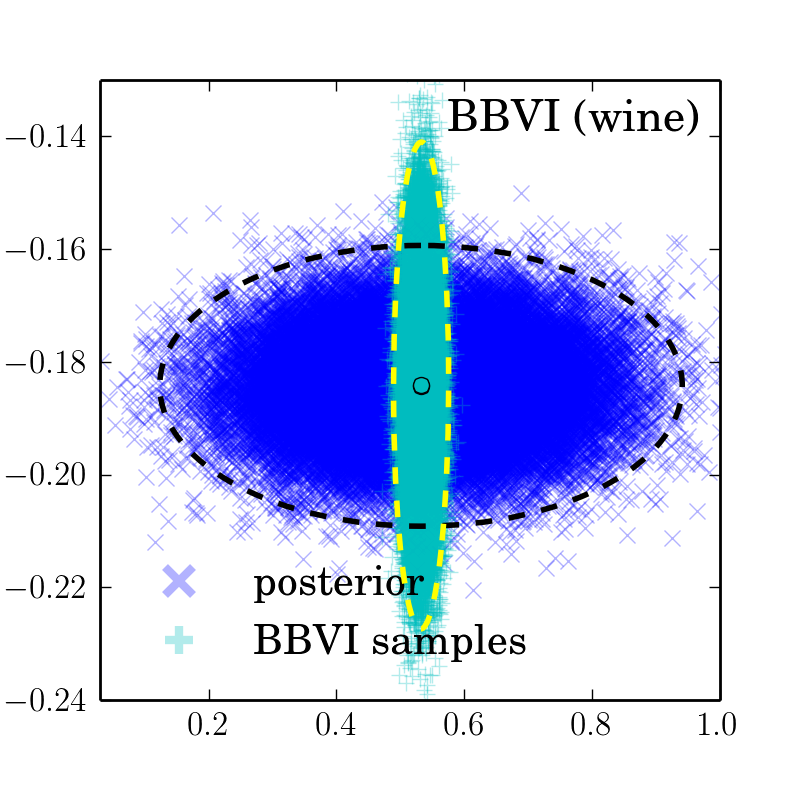}
  \includegraphics[width=.32\linewidth]{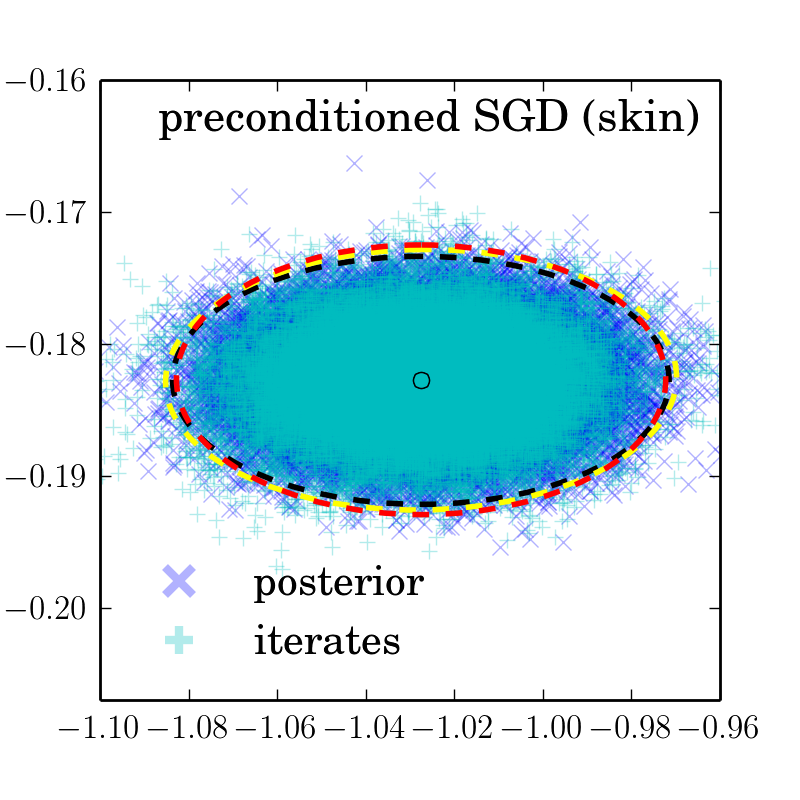}
  \includegraphics[width=.32\linewidth]{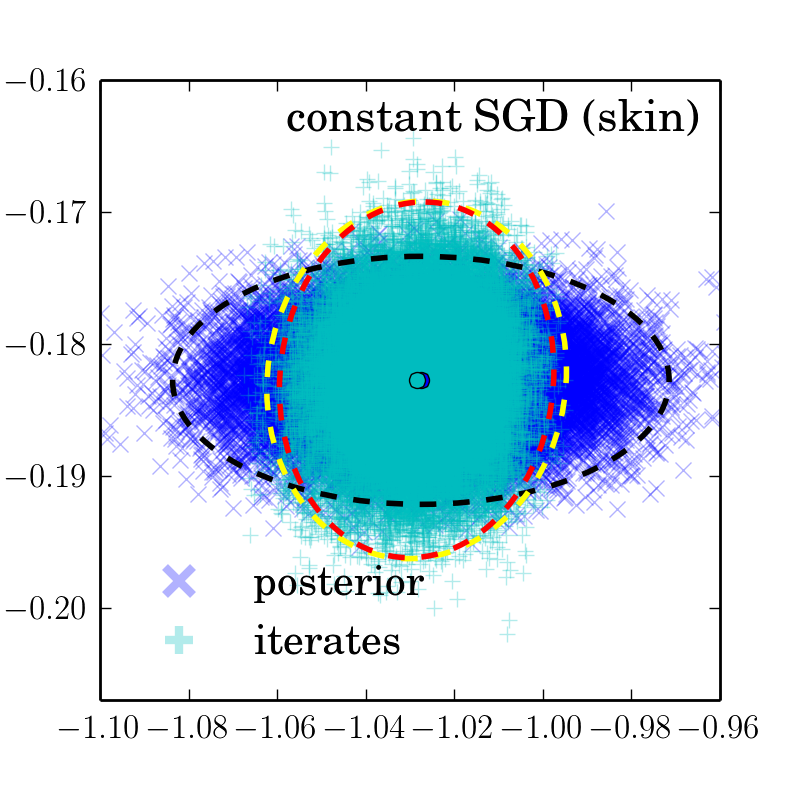}
  \includegraphics[width=.32\linewidth]{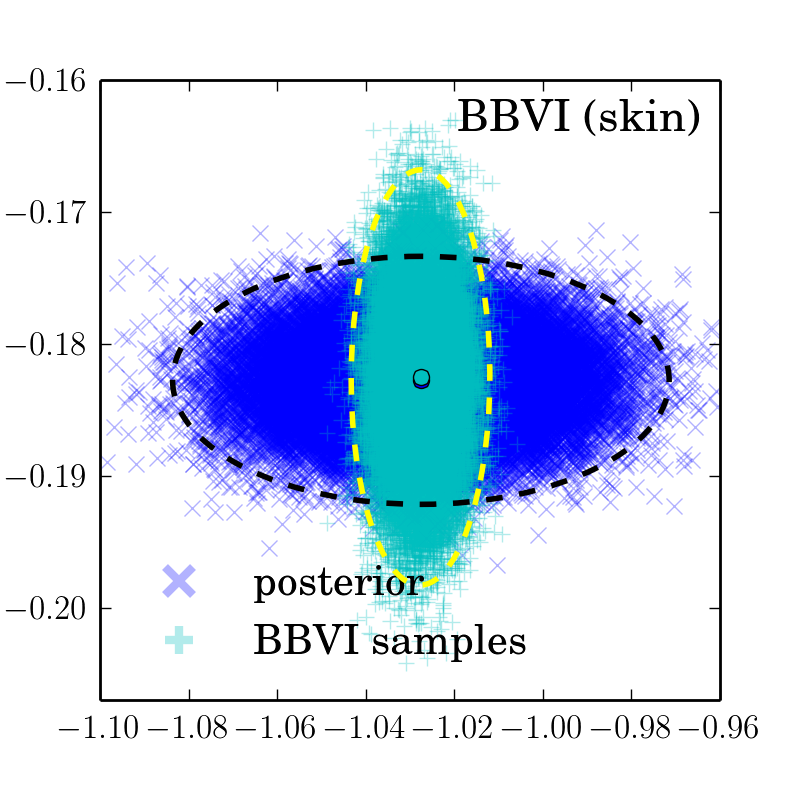}
  \caption{
    Posterior distribution $f(\theta) \propto \exp\left\{-N{\cal L}(\theta)\right\}$ (blue)
    and stationary sampling distributions $q(\theta)$ of the iterates of SGD (cyan) or
    black box variational inference (BBVI) based on reparameterization gradients.
    Rows: linear regression (top) and
    logistic regression (bottom) discussed in Section~\ref{sec:experiments}. Columns:
    full-rank preconditioned constant SGD (left), constant SGD (middle),
    and BBVI{\small~\citep{kucukelbir2015automatic}} (right). We show projections
    on the smallest and largest principal component of the posterior. The plot also shows
    the empirical covariances (3 standard deviations) of the posterior (black), the covariance of
    the samples (yellow), and their prediction (red) in
    terms of the Ornstein-Uhlenbeck process, Eq.~\ref{eq:stationaryvariance}.
  }
  \label{fig:regression}
\end{figure}

\begin{figure}[htbp] \centering
  \includegraphics[width=0.33\linewidth]{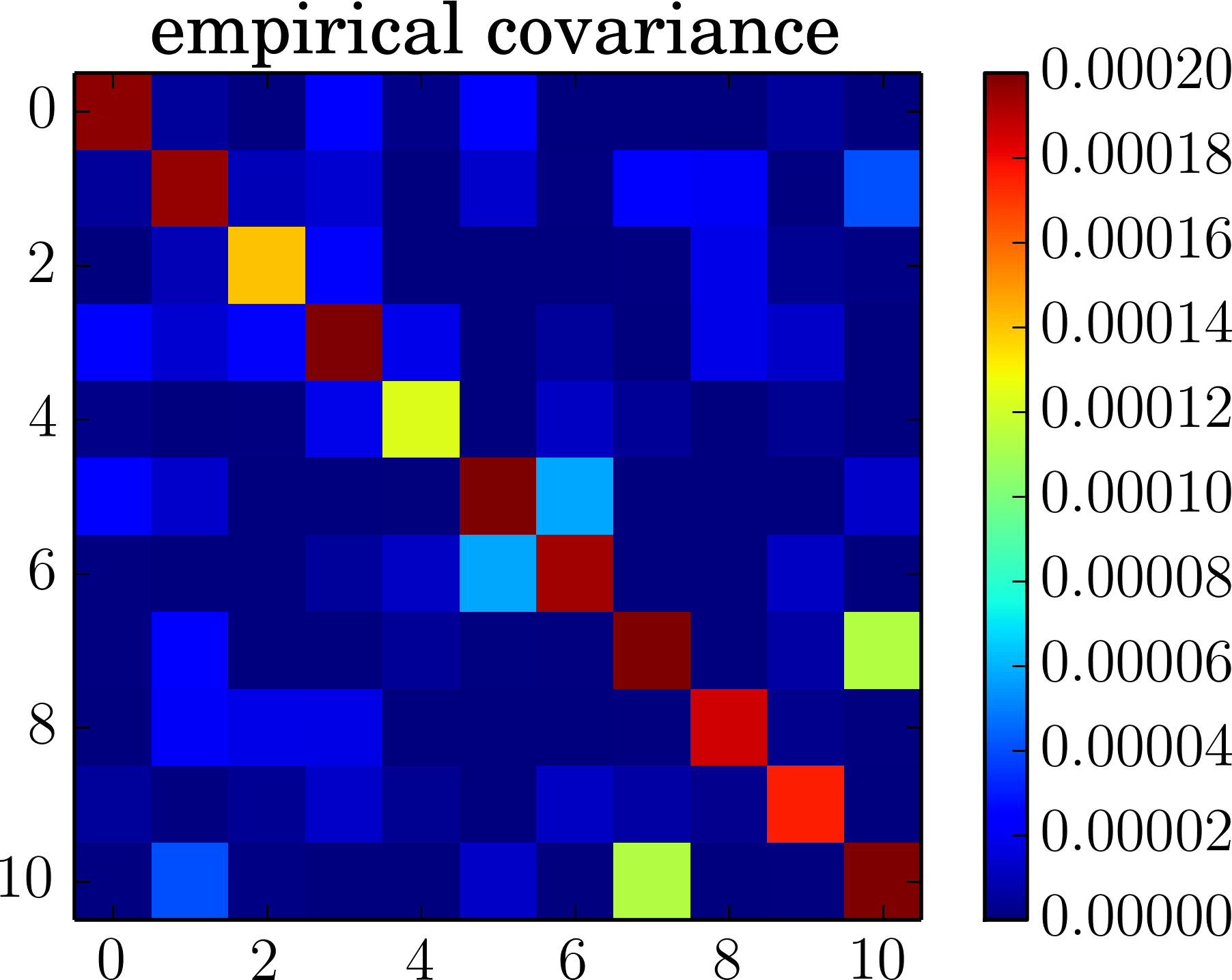}
  \includegraphics[width=0.33\linewidth]{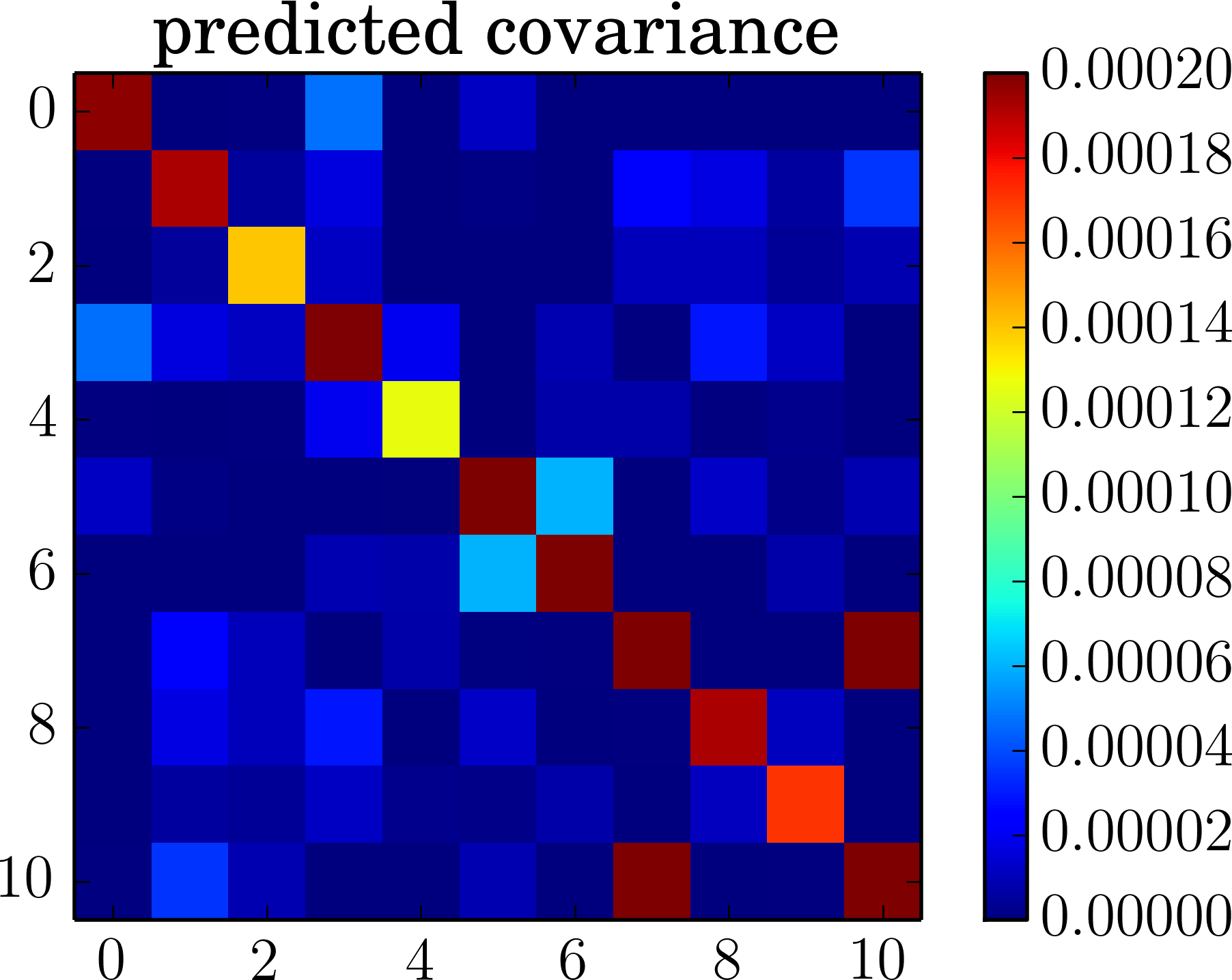}
  \caption{Empirical and predicted covariances of the iterates of
    stochastic gradient descent, where the prediction is based on
    Eq.~\ref{eq:stationaryvariance}.  We used linear regression on the
    wine quality data set as detailed in Section~\ref{sec:toy_data}. }
  \label{fig:covariance}
\end{figure}

Fig.~\ref{fig:regression} shows an example. Here we illustrate two
Bayesian posteriors---from a linear regression problem (top) and a logistic
regression problem (bottom)---along with iterates from a constant SGD
algorithm. In these figures, we set the parameters of the optimization
to values that minimize the Kullback-Leibler (KL) divergence between
the stationary distribution of the OU process and the
posterior---these results come from our theorems below. The left plots
optimize both a preconditioning matrix and the step size;
the middle plots optimize only the step size (both are outlined in Section~\ref{sec:SGD_bare}).
We can see that the
stationary distribution of constant SGD can be made close to the exact
posterior.

Fig.~\ref{fig:covariance} also compares the empirical covariance of the
iterates with the predicted covariance in terms of
Eq.~\ref{eq:stationaryvariance}. The close match supports the
assumptions of Section~\ref{sec:formalism}.

We will use this perspective in three ways.  First, we develop optimal
conditions for constant SGD to best approximate the
posterior, connecting to well-known results around adaptive learning
rates and preconditioners~\citep{duchi2011adaptive,tieleman2012lecture}.
Second, we propose an algorithm for hyperparameter optimization based on constant SGD.
Third, we use it to analyze the stationary distribution
of stochastic gradient descent with momentum~\citep{polyak1964some}.

\subsection{Constant Stochastic Gradient Descent}
\label{sec:SGD_bare}

First, we show how to tune constant SGD's parameters to minimize 
KL divergence to the posterior; this is a type of variational
inference~\citep{Jordan:1999}. 
The analysis leads to three versions of constant SGD---one with a constant step size,
one with a full preconditioning matrix, and one with a diagonal
preconditioning matrix.  Each yields samples from an approximate
posterior, and each reflects a different tradeoff between efficiency and accuracy.
Finally, we show how to use these 
algorithms to learn hyperparameters.

Assumption 4 from Section~\ref{sec:formalism} says that the posterior is
approximately Gaussian in the region that the stationary distribution
focuses on,
\begin{align}
\label{eq:posterior}
f(\theta)  \propto   \exp\left\{-\textstyle\frac{N}{2}\theta^\top A \theta \right\}.
\end{align}
The scalar $N$ corrects the averaging in equation \ref{eq:loss-one}.
Furthermore,
in this section we will consider a more general SGD scheme that may involve
a \emph{preconditioning matrix} $H$ instead of a scalar learning rate $\epsilon$:
\begin{align}
  \theta_{t+1} = \theta_t - H \hat{g}_S(\theta(t)). \nonumber
\end{align}

We will set the parameters of SGD to
minimize the KL divergence between the stationary distribution
$q(\theta)$ (Eqs.~\ref{eq:stationarydistribution} and
\ref{eq:stationaryvariance}) and the posterior $f(\theta)$ (Eq.~\ref{eq:posterior}). This involves
the the learning rate $\epsilon$
or more generally the preconditioning matrix $H$
and the minibatch size $S$:
\begin{align}
\{H^{*},S^{*}\}  = \arg\min_{H,S} KL(q \mid\mid f). \nonumber
\end{align}
First, consider a scalar learning rate $\epsilon$
(or a trivial preconditioner $H = \epsilon {\bf I}$).
The distributions $f(\theta)$ and $q(\theta)$ are both
Gaussians. Their means coincide, at the minimum of the loss, and so
their KL divergence is
\begin{align}
KL(q  \,||\, f ) & =  -\E_{q}[\log f(\theta)] +  \E_{q}[\log q(\theta)] \nonumber \\
& =  \textstyle\frac{1}{2}\left( N \E_q[\theta^\top A\theta] - \log |NA| -
  \log |\Sigma| - D \right)\nonumber \nonumber \\
& =  \textstyle\frac{1}{2}\left( N {\rm Tr}(A\Sigma) - \log |NA| -
  \log |\Sigma| - D \right),\nonumber
\end{align} where $|\cdot|$ is the
determinant and $D$ is the dimension of $\theta$.

We suggest three variants of constant SGD that generate samples from
an approximate posterior.

\begin{theorem}[Constant SGD] 
Under Assumptions
  A1-A4, the constant learning rate that minimizes KL divergence
  from the stationary distribution of constant SGD to the posterior is
\begin{align}
\label{eq:version_1}
\epsilon^*  =   \textstyle2\frac{S}{N}\frac{D}{{\rm Tr}(BB^\top)}.
\end{align}
\label{thm:one}
\end{theorem}

\begin{proof}
To prove this claim, we face the problem that the covariance of the
stationary distribution depends indirectly on $\epsilon$ through
Eq.~\ref{eq:stationaryvariance}. Inspecting this equation reveals
that $\Sigma_0 \equiv \textstyle\frac{S}{\epsilon}\Sigma$ is
independent of $S$ and $\epsilon$. This simplifies the entropy term
$\log |\Sigma| = D \log (\epsilon/S) + \log |\Sigma_0|$. Since
$\Sigma_0$ is constant, we can neglect it when minimizing KL
divergence.

We also need to simplify the term ${\rm Tr}(A\Sigma)$, which still
depends on $\epsilon$ and $S$ through $\Sigma$. To do this, we again use
Eq.~\ref{eq:stationaryvariance}, from which follows that
${\rm Tr}(A\Sigma) = \frac{1}{2}({\rm Tr}(A\Sigma)+{\rm Tr}(\Sigma
A)) = \frac{\epsilon}{2S}{\rm Tr}(BB^\top)$.
The KL divergence is therefore, up to constant terms,
\begin{align}
KL(q  || f ) \stackrel{\mathrm{c}}{=} \textstyle\frac{\epsilon\, N}{2S} {\rm Tr}(BB^\top) - D \log (\epsilon/S).
\end{align}
Minimizing KL divergence over $\epsilon/S$ results in Eq.~\ref{eq:version_1} for the optimal learning rate. 
\end{proof}

Theorem~\ref{thm:one} suggests that the learning rate should be chosen inversely
proportional to the average of diagonal entries of the noise covariance,
and proportional to the ratio between the minibatch size and dataset
size.

Since both the posterior $f$ and the variational distribution $q$ are Gaussian, one might wonder
if also the reverse KL divergence is viable. 
While the KL divergence can be computed up to a constant, we cannot remove its dependence on the unknown stationary distribution
$\Sigma$ using Eq.~\ref{eq:stationaryvariance}, unless $A$ and $\Sigma$ commute.  This setup is discussed in Appendix \ref{sec:app_reverseKL}.

Instead of using a scalar learning rate, we now consider
a positive-definite preconditioning matrix $H$. This gives us more tuning parameters
to better approximate the posterior.

\begin{theorem}[Preconditioned constant SGD]
The preconditioner for constant SGD that
minimizes KL divergence from the stationary distribution to the
posterior is 
\begin{align}
H^* = \textstyle \frac{ 2S}{N} (BB^\top)^{-1}.
\label{eq:version_2}
\end{align}
\label{thm:two}
\end{theorem}

\begin{proof}
To prove this, we need the Ornstein-Uhlenbeck process which corresponds
to preconditioned SGD. Replacing the constant learning rate in  Eq.~\ref{eq:OU-process} with a positive-definite preconditioning matrix $H$ results in
\begin{align}
d \theta(t)  =  - HA \,\theta(t) dt + \invS H B \, dW(t). \nonumber
\end{align}
All our results carry over after substituting $\epsilon A  \leftarrow  HA, \;  \epsilon B  \leftarrow  HB$.
Eq.~\ref{eq:stationaryvariance}, after the transformation and
multiplication by $H^{-1}$ from the left, becomes
\begin{align}
\label{eq:inproofthm2}
A \Sigma   + H^{-1} \Sigma A H= \textstyle\frac{1}{S} BB^\top H.
\end{align}
Using the cyclic property of the trace, this implies that
\begin{align}
\label{eq:thrm2_energy}
{\rm Tr}(A\Sigma) = \textstyle\frac{1}{2}({\rm Tr}(A\Sigma) + {\rm Tr}(H^{-1}A\Sigma H)) = \textstyle\frac{1}{2S}{\rm Tr}(BB^\top H).
\end{align}
Consider now the log determinant term, $\log |\Sigma|$, which still has an implicit dependence on $H$. We first define $Q = \Sigma H^{-1}$, hence $Q^\top = H^{-1}\Sigma$ since $\Sigma$, $H$ and $H^{-1}$ are symmetric. Eq.~\ref{eq:inproofthm2} can be written as $A \Sigma H^{-1} +  H^{-1} \Sigma A = \textstyle\frac{1}{S} BB^\top$, which is equivalent to $QA + AQ^\top = \textstyle\frac{1}{S} BB^\top$. Thus, we see that $Q$ is independent of $H$. The log determinant term is up to a constant $\log| \Sigma| = \log|H| + \log |Q|. $
Combining Eq.~\ref{eq:thrm2_energy} with this term, the  KL divergence is up to a constant
\begin{align}
\label{eq:mod_KL}
KL(q  || f ) & \stackrel{\mathrm{c}}{=}  
		  \textstyle\frac{N}{2S}{\rm Tr}(BB^\top H) +\log|H| + \log |Q|.
\end{align}
Taking derivatives with respect to the entries of $H$ results in Eq.~\ref{eq:version_2}.
\end{proof}

In high-dimensional applications, working with large dense matrices is impractical. In those settings
we can constrain the preconditioner to be diagonal.
The following corollaries are based on the proof of Theorem~\ref{thm:two}:

\begin{corollary}
The optimal diagonal preconditioner for SGD that
minimizes KL divergence to the posterior is
$H_{kk}^* = \textstyle \frac{ 2S}{N BB_{kk}^\top}$.
\end{corollary}
\begin{proof}
This follows from Eq.~\ref{eq:mod_KL}, where we  restrict the preconditioning matrix to be diagonal.
\end{proof}

\begin{corollary}
Under assumptions A1-A4, preconditioning with the full inverse noise covariance as in Theorem~\ref{thm:two}
results in samples from the exact posterior.
\end{corollary}

\begin{proof}
Consider Eq.~\ref{eq:inproofthm2}. Inserting $H = \textstyle \frac{ 2S}{N} (BB^\top)^{-1}$ results
in $A\Sigma + H^{-1}\Sigma A H = \textstyle{\frac{2}{N}} {\bf I}$ which is solved by $\Sigma = A^{-1}/N$
which is the posterior covariance.
\end{proof}

We showed that the optimal diagonal preconditioner is the inverse of
the diagonal part of the noise matrix.  Similar preconditioning
matrices have been suggested earlier in optimal control theory based
on very different arguments, see~\citep{widrow1985adaptive}.  Our
result also relates to AdaGrad and its
relatives~\citep{duchi2011adaptive, tieleman2012lecture}, which also
adjust the preconditioner based on the square root of the diagonal entries of the noise covariance.
In Appendix~\ref{sec:sqrt}
we derive an optimal global learning rate for AdaGrad-style diagonal preconditioners.
In Section~\ref{sec:experiments}, we compare three versions of
constant SGD for approximate posterior inference: one with a scalar
step size, one with a dense preconditioner, and one with a diagonal
preconditioner.

\parhead{Remark on estimating the noise covariance.} In order to use our theoretical insights in practice, we need to estimate the 
stochastic gradient noise covariance $C \equiv BB^\top$. We do this in an online manner. 
As before, let $g_t$ be the full gradient, $\hat{g}_{S,t}$
be the stochastic gradient of the full minibatch, and $\hat{g}_{1,t}$ be the stochastic gradient of the first sample in the minibatch at time $t$ (which has a much larger variance
if $S \gg 1$).
For large $S$, we can approximate $g_t \approx \hat{g}_{S,t}$, and thus obtain an estimator of the noise covariance by $(\hat{g}_{1,t}-\hat{g}_{S,t})(\hat{g}_{1,t}-\hat{g}_{S,t})^\top$.
Following~\citet{ahn2012bayesian}, we can now build an online estimate $C_t$ that approaches $C$ by the following recursion,
\be
\label{eq:estimate-noise}
C_{t} = (1-\kappa_t) C_{t-1} + \kappa_t (\hat{g}_{1,t}-\hat{g}_{S,t})(\hat{g}_{1,t}-\hat{g}_{S,t})^\top.
\ee
Above, $\kappa_t$ is a decreasing learning rate.
  \citet{ahn2012bayesian} have proven that such an online average
  converges to the noise covariance in the optimum at long times (provided that $\kappa_t \sim 1/t$ and that $N$ is sufficiently large).
  We found that this online estimator works well in practice, even though our theoretical assumptions would
  require preconditioning SGD with the true noise covariance in finite time. Regarding the computational overhead
  of this procedure, note that similar online estimates of the gradient noise are carried out in adaptive SGD
  schemes such as AdaGrad~\citep{duchi2011adaptive} or RMSProp~\citep{tieleman2012lecture}.
  When using a diagonal approximation to the noise covariance,
  the costs are proportional in the number of dimensions
  and mini-batch size; this efficiency means that the online estimate does not spoil the efficiency of SGD. Full preconditioning
  scales quadratically in the dimension and is therefore impractical in many real-word setups.

\subsection{Constant SGD as Variational EM}
\label{sec:hyperparameter_theory}

Consider a supervised probabilistic model with 
joint distribution $p(y,\theta|x,\lambda) = p(y|x,\theta) p(\theta | \lambda)$ that factorizes into a likelihood and prior, respectively. Our goal is to find optimal hyperparameters $\lambda$. 
Jointly point-estimating $\theta$ and $\lambda$ by following gradients of the log joint leads to overfitting or degenerate solutions. 
This can be prevented in a Bayesian approach, where we treat the parameters $\theta$ as latent variables. 
In Empirical Bayes (or type-II maximum likelihood),
we maximize the \emph{marginal} likelihood of the data, integrating out
the main model parameters:
\begin{equation}
\begin{split}
\textstyle \lambda^\star = \arg\max_{\lambda}\log
p(y|x,\lambda)=\arg\max_{\lambda}\log\int_\theta p(y,\theta|x,\lambda)d\theta.
\nonumber
\end{split}
\end{equation}
When this marginal log-likelihood is intractable, a common approach is
to use \emph{variational expectation-maximization (VEM)}
\citep{Bishop:2006}, which iteratively optimizes a variational lower bound on the
marginal log-likelihood over $\lambda$.  If we approximate the
posterior $p(\theta|x,y,\lambda)$ with some distribution $q(\theta)$,
then VEM tries to find a value for $\lambda$ that maximizes the
expected log-joint probability $\mathbb{E}_q[\log
  p(\theta,y|x,\lambda)]$.

Constant SGD gives rise to a simple VEM algorithm that applies to a large class of differentiable models.
Define ${\cal L}(\theta,\lambda) = -\log p(y,\theta \g x,\lambda) $. 
If we interpret the stationary distribution of SGD as a variational
approximation to a model's posterior,  we can justify 
following a stochastic gradient descent scheme on both $\theta$ and $\lambda$:
\begin{equation}
\theta_{t+1}  = \theta_t  - \epsilon^*\nabla_\theta {\cal L}(\theta_t,\lambda_t);\quad
\lambda_{t+1}  = \lambda_t - \rho_t \nabla_\lambda {\cal
  L}(\theta_t,\lambda_t). 
\label{eq:variationalEM}
\end{equation}
While the update for $\theta$ uses the optimal constant learning rate $\epsilon^*$ and therefore
samples from an approximate posterior, the $\lambda$ update uses a decreasing learning
rate $\rho_t$ and therefore converges to a local optimum.
The result is a type of VEM algorithm.

We stress that the optimal constant learning rate $\epsilon^*$ is
not unknown, but can be estimated based on an online estimate of the noise covariance $C \equiv BB^\top$, 
as given in Eq.~\ref{eq:estimate-noise}.  In Section~\ref{sec:experiments} we show that gradient-based hyperparameter
learning is a cheap alternative to cross-validation.

\subsection{Stochastic Gradient with Momentum}
\label{sec:momentum}
The continuous-time formalism also allows us to explore extensions
of classical SGD. One of the most popular methods is stochastic gradient with momentum~\citep{polyak1964some,sutskever2013importance}.
Here, we give a version of this algorithm that allows us to sample from an approximate posterior.

SGD with momentum doubles the dimension of parameter space in introducing an additional \emph{momentum} variable $\v$ that has the same dimension as $\theta$.
Position and momentum are coupled in such a way that the algorithm keeps
memory of some of its past gradients.  The updates of SGD with momentum are
\be
\v(t+1) & = & (1-\mu) \v(t) - \epsilon \, \hat{g}_S(\theta(t))\nonumber  \\
\theta(t+1) & = & \theta(t) + \v(t+1). \nonumber
\ee
This involves the damping coefficient $\mu \in [0,1]$.
For $\mu=1$ (infinite damping or overdamping), the momentum information gets lost and we recover SGD. 

As before we assume a quadratic objective ${\cal L} = \frac{1}{2}\theta^\top A \theta$. 
Going through the same steps A1-A4 of Section~\ref{sec:formalism} that allowed us to
derive the Ornstein-Uhlenbeck process for SGD, we find
\be
\label{eq:momentum_continuous}
d \v        & = &  -\mu \v dt - \epsilon A \theta dt + \invS \epsilon B \,dW,\\
d \theta & = &  \v dt. \nonumber
\ee
We solve this set of stochastic equations asymptotically for the long-time limit. 
We reformulate the stochastic equation in terms of coupled equations for the moments.
(This strategy was also used by~\cite{li2015dynamics} in a more
restricted setting). 
The first moments of the set of coupled stochastic differential equations give
\be
d\E[\v]  =   -\mu  \E[\v] dt - \epsilon A \E[\theta] dt, \quad  d\E[\theta]  =   \E[\v]dt. \nonumber
\ee
Note we used that expectations commute with the differential operators. These deterministic equations have the simple solution $\E[\theta] = 0$ and $\E[\v]=0$,
which means that the momentum has expectation zero and the expectation of the position variable converges to the optimum (at $0$).

In order to compute the stationary distribution, we derive and
solve similar equations for the second moments. These calculations are
carried out in Appendix~\ref{sec:app_momentum}, where we derive the
following conditions:
\be
\E[\v\v^\top] & = &  \textstyle{\frac{\epsilon}{2}} \E[\theta \theta^\top] A +  \textstyle{\frac{\epsilon}{2}}  A \E[\theta \theta^\top], \\
 \mu \E[\v\v^\top] & = & \textstyle{\frac{\epsilon^2}{2S}} BB^\top. \nonumber
\label{eq:fdt}
\ee
Eq.~\ref{eq:fdt} relate to equilibrium thermodynamics, where 
$\E[\v\v^\top]$ is a matrix of expected kinetic energies, while
$\half(\E[\theta \theta^\top] A + A \E[\theta \theta^\top] )$ has the
interpretation of a matrix of expected potential energies.
The first equation says that energy conservation holds in expectation,
which is the case for an equilibrium system which exchanges
energy with its environment.
The second equation relates the covariance of the velocities with
$BB^\top$, which plays the role of a
matrix-valued temperature. This is known as the
fluctuation-dissipation theorem~\citep{Nyquist}. Combining both
equations and using $\Sigma = \E[\theta \theta^\top]$ yields
\be
\Sigma A + A \Sigma  & = & \textstyle{ \frac{\epsilon}{\mu S}} BB^\top.
\ee
This is exactly Eq.~\ref{eq:stationaryvariance} of SGD without
momentum, however, with the difference that the noise
covariance is re-scaled by a factor $ \frac{\epsilon}{ \mu S}$
instead of $ \frac{\epsilon}{S}$.

We have shown that $\epsilon$, $S$, and $\mu$ play similar roles: only the combination $ \textstyle{ \frac{\epsilon}{ \mu S}}$ affects the
KL divergence to the posterior. Thus, no single optimal constant learning rate exists---many combinations of $\epsilon$, $\mu$, and $S$ can
yield the same stationary distribution. But different choices of these parameters
will affect the dynamics of the Markov chain. For example, \citet{sutskever2013importance}
observe that, for a given effective learning rate $\frac{\epsilon}{\mu S}$,
using a smaller $\mu$ sometimes makes the \emph{discretized} dynamics of SGD more stable.
Also, using very small values of $\mu$ while holding $\frac{\epsilon}{\mu}$ fixed
will eventually increase the autocorrelation time of the Markov chain (but
this effect is often negligible in practice).

\section{Analyzing Stochastic Gradient MCMC Algorithms}
\label{sec:MCMC}
We have analyzed well-known stochastic optimization algorithms such as SGD, preconditioned SGD and SGD with momentum. 
We now investigate Bayesian sampling algorithms. A large class of modern MCMC
methods rely on stochastic gradients~\citep{welling2011bayesian, ahn2012bayesian, chen2014stochastic, ding2014bayesian, ma2015complete}.
The central idea is to add artificial noise to the stochastic gradient to asymptotically sample from the true posterior. 

In practice, however, the algorithms are used in combination with several heuristic approximations, such as non-vanishing learning rates or diagonal approximations to the Hessian. 
In this section, we use the variational perspective to quantify these biases and help understanding these algorithms better under more realistic assumptions.

\subsection{SGLD with Constant Rates}
To begin with, we we analyze the well-known Stochastic Gradient Langevin Dynamics by~\citet{welling2011bayesian}.
This algorithm has been carefully analyzed in the long time limit where the stochastic gradient noise vanishes as the learning
rate goes to zero, and where mixing becomes infinitely slow~\citep{sato2014approximation}.
Here we analyze an approximate version of the algorithm,
SGLD with a constant learning rate. We are interested in the stationary distribution of this approximate inference algorithm.

The discrete-time process that describes Stochastic Gradient Langevin dynamics is
\be
\theta_{t+1} & = & \theta_t - \textstyle{\frac{\epsilon}{2}} N\hat{\nabla}_\theta {\cal L}(\theta_t) + \sqrt{\epsilon} \, V(t),\nonumber
\ee
where $V(t)\sim {\cal N}(0,{\bf I})$ is a vector of independent Gaussian noises. Following assumptions 1--4 of Section~\ref{sec:formalism}, $V$ becomes the Wiener noise $dV$ and the corresponding continuous-time process is
\be
d\theta & = & -\half \epsilon N A\theta dt + \sqrt{\epsilon} dV +  \epsilon \invS N B \,dW.\nonumber
\ee
Above, $dV$ and $dW$ are vectors of independent Wiener noises, i.e. $\E[dW dV^\top]=\E[dV dW^\top]=0$.
The analog of Eq.~\ref{eq:stationaryvariance} is then
\be
\half N (A \Sigma +  \Sigma A) & = & {\bf I} + \textstyle{\frac{\epsilon N}{2S}} BB^\top. \nonumber
\ee
In the limit of $\epsilon \rightarrow 0$, we find that $\Sigma^{-1} = NA$, meaning the stationary distribution becomes identical to the posterior.
However, for non-zero $\epsilon$, there are discrepancies.
These correction-terms are positive. This shows that the posterior covariance is generally overestimated by Langevin dynamics,
which can be attributed to non-vanishing learning rates at long times.

\subsection{Stochastic Gradient Fisher Scoring}
\label{sec:fisher}

We now investigate Stochastic Gradient Fisher
Scoring~\citep{ahn2012bayesian}, a scalable Bayesian
MCMC algorithm.  We use the variational perspective to rederive the Fisher
scoring update and identify it as optimal. We also analyze the
sampling distribution of the truncated algorithm, one with diagonal
preconditioning (as it is used in practice), and quantify the bias
that this induces.

The basic idea here is that the stochastic gradient is preconditioned
and additional noise is added to the updates such that the algorithm
approximately samples from the Bayesian posterior.  More precisely,
 the update can be cast into the following form:
\begin{align}
\theta(t+1)  =  \theta(t) - \epsilon H \, \hat{g}(\theta(t)) + \sqrt{\epsilon} H E \, W(t).
\label{eq:SGFS}
\end{align}
The matrix $H$ is a preconditioner and $E W(t)$ is Gaussian noise; we
control the preconditioner and the covariance $EE^\top$ of the
noise. Stochastic gradient Fisher scoring suggests a preconditioning
matrix $H$ that leads to samples from the posterior even if the
learning rate $\epsilon$ is not asymptotically small.  We show here
that this preconditioner follows from our variational analysis.

\begin{theorem}[Stochastic Gradient Fisher Scoring]
Under Assumptions A1-A4, the positive-definite preconditioner $H$  in Eq.~\ref{eq:SGFS} that
minimizes KL divergence from the stationary distribution of SGFS to the posterior is
\begin{align}
\label{eq:fisher_H}
\textstyle H^{*}  =  \frac{2}{N}( \epsilon BB^\top + EE^\top)^{-1}.
\end{align}
\label{thm:three}
\end{theorem}

\begin{proof}
To prove the claim, we go through the steps of Section~\ref{sec:formalism} to derive the corresponding Ornstein-Uhlenbeck process,
$d \theta(t)  =  -\epsilon HA \theta(t) dt + H \left[  \epsilon B + \sqrt{\epsilon} E\right] dW(t).$
For simplicity, we have set the minibatch size $S$ to $1$.
In Appendix~\ref{sec:sgfs}, we derive the following KL divergence between the posterior and the sampling distribution:
$$
KL(q||p)  =  - \textstyle\frac{N}{4}{\rm Tr}  (H(\epsilon BB^\top + EE^\top)) + \textstyle\frac{1}{2} \log |T| + \textstyle\frac{1}{2} \log |H| + \textstyle\frac{1}{2} \log |NA| + \textstyle\frac{D}{2}.
$$
($T$ is constant with respect to $H$, $\epsilon$, and $E$.)
We can now minimize this KL divergence over the parameters $H$ and $E$.
When $E$ is given, minimizing over $H$ gives Eq.~\ref{eq:fisher_H}.
\end{proof}

Eq.~\ref{eq:fisher_H} not only minimizes the KL
divergence, but makes it $0$, meaning that the stationary
sampling distribution \emph{is} the posterior. This solution
corresponds to the suggested Fisher Scoring update in the idealized
case when the sampling noise distribution is estimated
perfectly~\citep{ahn2012bayesian}. Through this update, the algorithm
thus generates posterior samples without decreasing the
learning rate to zero.  (This is in contrast to Stochastic Gradient
Langevin Dynamics by~\citet{welling2011bayesian}.)

In practice, however, SGFS is often used with a diagonal approximation of
the preconditioning matrix~\citep{ahn2012bayesian,ma2015complete}.
However, researchers have not explored how the stationary distribution
is affected by this truncation, which makes the algorithm only
approximately Bayesian.  
We can quantify its deviation from the exact posterior and we derive the optimal 
diagonal preconditioner, which follows from the KL divergence in Theorem~\ref{thm:three}:

\begin{corollary}
When approximating the Fisher scoring preconditioner by a diagonal matrix $H_{kk}^{*}$ or a scalar $H^*_{scalar}$, respectively,
then
$$H_{kk}^{*}  =  \frac{2}{N}( \epsilon BB^\top_{kk} + EE^\top_{kk})^{-1} {\textrm \quad  and \quad}
H^*_{scalar} = \frac{2D}{N} (\sum_k [ \epsilon BB^\top_{kk} + EE^\top_{kk}])^{-1}.$$
\end{corollary}

Note that we have not made any assumptions about the noise covariance $E$.
We can adjust it in favor of a more stable algorithm.
For example, in the interests of stability we might want to 
set a maximum step size
$h^{max}$, so that $H_{kk}\leq h^{max}$ for all $k$. We can 
adjust $E$ such that $H_{kk} \equiv h^{max}$ in Eq.~\ref{eq:fisher_H} becomes independent of $k$.
Solving for $E$ yields
$EE^\top_{kk}  =  \textstyle\frac{2}{h^{max}N} -  \epsilon BB^\top_{kk}$.

Hence, to keep the learning rates bounded in favor of
stability, one can inject noise in dimensions where the variance of
the gradient is too small.  This guideline is opposite to the
advice of \citet{ahn2012bayesian} to choose $B$ proportional to $E$,
but follows naturally from the variational analysis.

An additional benefit of SGFS over simple constant SGD is that the sum
of gradient noise and Gaussian noise will always look ``more
Gaussian'' than the gradient noise on its own. An extreme case is when
the gradient covariance $BB^\top$ is not full rank; in this situation
injecting full-rank Gaussian noise could prevent degenerate behavior.

\section{A Bayesian View on Iterate Averaging}
\label{sec:iterate}
We now apply our continuous-time analysis to the technique of iterate averaging \citep{polyak1992acceleration}.
Iterate averaging was derived as an optimization algorithm, and we analyze this algorithm in Section~\ref{sec:iterate_opt}
from this perspective. In Section~\ref{sec:iterate_bayes} we show that iterate averaging can also be used as a Bayesian sampling algorithm.

\subsection{Iterate Averaging for Optimization}
\label{sec:iterate_opt}

Iterate averaging further refines SGD's estimate of the optimal parameters
by averaging the iterates produced by a series of SGD steps.
\citet{polyak1992acceleration} proved the remarkable result that
averaged SGD achieves the best possible convergence rate among
stochastic gradient algorithms\footnote{To be more precise, averaged
SGD is optimal among methods that only have access to a stochastic
gradient oracle---if more is known about the source of the noise
then sometimes better rates are possible \citep[e.g.,][]{johnson2013accelerating,
defazio2014saga}.},
including those that use second-order information such as Hessians.
This implies that the convergence speed of iterate averaging cannot be improved
when premultiplying the stochastic gradient with any preconditioning matrix.
We use stochastic calculus 
to show that the stationary distribution of iterate averaging is the same for any constant
preconditioner. Based on slightly stronger-than-usual assumptions, we give a short proof on why this result holds.

To eliminate asymptotic biases, iterate averaging usually requires a slowly decreasing learning rate. We consider a
simplified version with a constant rate, also analyzed in~\citep{zhang2004solving,nemirovski2009robust}. 

\parhead{Algorithm.} Before we begin iterate averaging, we assume that
we have run ordinary constant SGD for long enough that it has reached
its stationary distribution and forgotten its intialization. In this
scenario, we use iterate averaging to ``polish'' the results obtained
by SGD.

Iterate averaging estimates the location of the minimum of $\mathcal{L}$ using a
sequence of stochastic gradients $\hat{g}_S$, and then computes an average of the iterates in an online manner,
\be
\theta_{t+1} & = & \theta_t - \epsilon \hat{g}_S(\theta_t),  \\
\hat{\mu}_{t+1} & = & \textstyle \frac{t}{t+1}\, \hat{\mu}_{t} +  \frac{1}{t+1}\, \theta_{t+1}. \nonumber 
\ee
After $T$ stochastic gradient steps and going over to continuous times, this average is
\begin{equation}
\begin{split}
\textstyle
\hat\mu \approx \textstyle\frac{1}{T}\int_0^{T} \theta(t) dt \equiv \hat\mu'.
\end{split}
\label{eq:estimatorcovariance}
\end{equation}
The average $\hat\mu$ and its approximation $\hat{\mu}'$ are random variables whose expected value is the
minimum of the objective. (Again, this assumes $\theta(0)$ is drawn from the SGD process's stationary distribution.)

The accuracy of this estimator $\hat{\mu}'$ after a fixed number of iterations $T$ is characterized
by its covariance matrix $D$. Using stochastic calculus, we can compute $D$ from the
autocorrelation matrix of the stationary distribution of the OU process, shown in Appendix~\ref{sec:iterate_proof}.
We prove that
\begin{equation}
\begin{split}
\textstyle
D \; \equiv\;  \E\left[ \hat{\mu}' \hat{\mu}'^\top \right] \; \approx
\;  \textstyle\frac{1}{\epsilon T} \left(\Sigma (A^{-1})^\top  +
  A^{-1}\Sigma\right).
\end{split}
\label{eq:iterate_autocorr}
\end{equation}
This approximation ignores terms of order $1/(T \epsilon)^2$. This term is small since we assume
that the number of iterations $T$ is much larger than the inverse learning
rate $1/\epsilon$. The covariance shrinks linearly over
iterations as we expect from~\cite{polyak1992acceleration}. We use
Eq.~\ref{eq:stationaryvariance} to derive 
\begin{align}
D \approx  \textstyle\frac{1}{TS}A^{-1}BB^\top (A^{-1})^\top. \nonumber
\end{align}

Note that this covariance depends only on the number of iterations $T$
times the minibatch size $S$, not on the step size $\epsilon$. Since
$TS$ is the total number of examples that are processed, this means
that this iterate averaging scheme's efficiency does not depend
on either the minibatch size or the step size, proven first in~\citep{polyak1992acceleration}.

We can make a slightly stronger statement. If we precondition the
stochastic gradients with a positive-definite matrix $H$ (for example,
the inverse of the Hessian evaluated at the minimum), it turns out
that the covariance of the estimator remains unchanged. The
resulting Ornstein-Uhlenbeck process is
$
d\theta = -HA\theta(t)dt + \invS H B\, dW(t).
$ 
The resulting stationary covariance $D'$ of preconditioned iterate averaging is the same as above:
\begin{equation}
\begin{split}
D' &\approx \textstyle\frac{1}{T}(HA)^{-1}(\invS HB) (\invS HB)^\top ((HA)^\top)^{-1}
\\ &= \textstyle\frac{1}{TS}A^{-1}H^{-1}HBB^\top H H^{-1}A^{-1}
\\ &= \textstyle\frac{1}{TS}A^{-1}BB^\top A^{-1}.
\end{split}
\end{equation}
Both the stationary distribution and the autocorrelation
matrix change as a result of the preconditioning, and these changes
exactly cancel each other out. 

This optimality of iterate averaging was first derived by~\cite{polyak1992acceleration},
using quasi-martingales. Our derivation is based on stronger assumptions, but is shorter.

\subsection{Finite-Window Iterate Averaging for Posterior Sampling}
\label{sec:iterate_bayes}
Above, we used the continuous-time formalism to quickly rederive known
results about the optimality of iterate averaging as an optimization
algorithm.
Next, as in Section~\ref{sec:consequences},
we will analyze iterate averaging as an algorithm for approximate posterior inference.

We will show that (under some optimistic assumptions),
iterate averaging requires exactly $N$ gradient calls to generate
one sample drawn from the exact posterior distribution,
where $N$ is the number of observations. That is,
there exist conditions under which iterate averaging generates one true posterior
sample per pass over the data.
This result is both exciting and discouraging;
it implies that, since our assumptions are all optimistic and
iterate averaging is known to saturate the Cram\'er-Rao bound,
\emph{no} black-box stochastic-gradient MCMC algorithm can
generate samples in time sublinear in the number of data points.

In addition to Assumptions 1--4 of Section~\ref{sec:formalism}, we need
an additional assumption for our theoretical considerations to hold:

\begin{assumption}
Assume that the sample size $N$ is large enough that the Bernstein-von Mises
theorem~\citep{le1986asymptotic} applies (hence the posterior is Gaussian).
Also assume that the observed dataset was drawn from the model
$p(y~|\theta)$ with parameter $\theta=0$.
Then $A=BB^\top$, that is, the Fisher information matrix equals the Hessian.
\end{assumption}
This simplifies
equation \ref{eq:stationaryvariance}:
\begin{align}
    A\Sigma + \Sigma A = {\textstyle \frac{\epsilon}{S}}BB^\top
    \stackrel{A5}{\Longrightarrow}  \Sigma = \frac{\epsilon}{2S}{\bf I}.
    \label{eq:bvmstationary} 
\end{align}
That is, the sampling distribution of SGD is isotropic.

\parhead{Stationary distribution.} We will consider the sampling
distribution of the \emph{average} of $T$ successive samples from the
stationary SGD process with step size $\epsilon$. Going to the
continuous-time OU formalism, we show in
Appendix~\ref{sec:iterate_proof} that the stationary covariance of the
iterate averaging estimator defined in Eq.~\ref{eq:iterate_autocorr} is
\begin{align}
\begin{split}
  D  &= \frac{1}{ST} A^{-1} +
    \frac{1}{\epsilon ST^2} U \Lambda^{-2}(e^{-\epsilon T\Lambda} - {\bf I}) U^\top,
\end{split}
\label{eq:stationarycovariance_averaging}
\end{align}
where $U$ is orthonormal, $\Lambda$ is diagonal, and $U\Lambda U^\top=A$
is the eigendecomposition of the Hessian $A$.
We have previously assumed that the posterior has covariance
$\frac{1}{N}A^{-1}$.
Thus, to leading order in the ratio $1/(\epsilon T \Lambda)$, the stationary
distribution of fixed-window iterate averaging is a scaled version of
the posterior.

If we choose $T=N/S$, so that we average the iterates of a single pass
through the dataset, then the iterate-averaging sampling
distribution will have approximately the same covariance as the posterior,
\begin{equation}
  \begin{split}
    D^\star &= \frac{1}{N} A^{-1} +
    \frac{S}{\epsilon}\frac{1}{N^2} U \Lambda^{-2}(e^{-\frac{\epsilon}{S}N\Lambda} - {\bf I}) U^\top
\\&= \frac{1}{N} U\Lambda^{-1}\left({\bf I} +
    \frac{S}{\epsilon}\frac{1}{N} \Lambda^{-1}(e^{-\frac{\epsilon}{S}N\Lambda} - {\bf I})\right) U^\top.
    \label{eq:polyakposterior}    
\end{split}
\end{equation}
$D^\star$ and $A^{-1}$ have identical eigenvectors, and their
eigenvalues differ by a factor that goes to zero as
$\frac{\epsilon}{S}$ becomes large. Conversely, as
$\frac{\epsilon}{S}$ approaches zero, all of these eigenvalues
approach $\frac{\epsilon}{2S}$ as in Eq.~\ref{eq:bvmstationary}. (This
can be shown by taking a second-order Maclaurin approximation
of $e^{-\frac{\epsilon N}{S}\Lambda}$.)

Our analysis gives rise to the Iterate Averaging Stochastic Gradient sampler (IASG),
described in Algorithm~\ref{alg:IASG}. We now investigate its approximation error
and efficiency.

\begin{algorithm}[t]
{\bf input:} averaging window $T=N/S$, number of samples $M$, input for SGD.\\
 \For{$t=1$ to {$M*T$}}{
  $\theta_{t} = \theta_{t-1} - \epsilon \, \hat{g}_S(\theta_{t-1})$; // perform an SGD step;\\
  \If{$t \, {\rm mod} \,T = 0$}{
  $\mu_{t/T} = \frac{1}{T}\sum_{t' = 0}^{T-1} \theta_{t-t'}$; // average the $T$ most recent iterates
  }
  }
  {\bf output:} return samples \, $\{\mu_1, \dots, \mu_M\}$.
 \caption{The Iterate Averaging Stochastic Gradient sampler (IASG)}
 \label{alg:IASG}
\end{algorithm} 

\parhead{Approximation error, step size, and minibatch size.}  We now
focus on the correction terms that lead to deviations between the
iterate averaging estimator's covariance $D^*$ and the posterior
covariance $\textstyle{\frac{1}{N}A^{-1}}$.

The analysis above tells us that we can ignore these correction terms
if we choose a large enough $\frac{\epsilon}{S}$. But the analysis in
previous chapters assumes that $\frac{\epsilon}{S}$ is small enough
that assumptions 1--4 hold. These considerations are in tension.

When is $\frac{\epsilon}{S}$ ``large enough''?
Eq.~\ref{eq:polyakposterior} shows that the relative error is largest
in the direction of the smallest eigenvalue $\lambdamin \equiv \min_k
\Lambda_{kk}$ of $A$ (corresponding to the least-constrained direction
in the posterior).  We will focus our analysis on the relative error
in this direction, which is given by
\begin{align*}
\begin{split}
\textstyle
\errmax \equiv
\frac{S}{\epsilon}\frac{1}{N\lambdamin}
(e^{-\frac{S}{\epsilon}N\lambdamin}-1).
\end{split}
\end{align*}
We are given $N$ and $\lambdamin$, but we can control
$\frac{\epsilon}{S}$. To make $\errmax$ small, we must therefore
choose $\frac{\epsilon}{S} > \frac{c}{N\lambdamin}$ for some
constant $c$.  So larger datasets and lower-variance posteriors let
us use smaller stepsizes.

It is hard to say in general how small $\frac{\epsilon}{S}$ needs to
be to satisfy assumptions 1--4. But if assumption 4 is satisfied
(i.e., the cost is approximately quadratic), then assumption 3 (no
discretization error) cannot hold if $\epsilon >
\frac{2}{\lambdamax}$. This is the step size at which the
discretized noise-free gradient descent process becomes unstable for
quadratic costs.  (We define $\lambdamax\equiv\max_k\Lambda_{kk}$
analogous to $\lambdamin$.)

Combining this observation with the analysis above, we see that this
iterate averaging scheme will not work unless
\begin{align*}
\begin{split}
\frac{2}{S\lambdamax} > \frac{\epsilon}{S} > \frac{c}{N\lambdamin}
\Rightarrow \frac{2}{c} > \frac{S}{N}\frac{\lambdamax}{\lambdamin}.
\end{split}
\end{align*}
That is, we need the dataset size $N$ to be large enough relative to
the condition number $\frac{\lambdamax}{\lambdamin}$ of the Hessian
$A$ if this simple iterate-averaging scheme is to generate good
posterior samples. If the condition number is large relative to $N$,
then it may be necessary to replace the scalar step size $\epsilon$
with a preconditioning matrix $H\approx A^{-1}$ to reduce the
effective condition number of the Hessian $A$.

\parhead{Efficiency.} 
Next, we theoreticaly analyze the efficiency with which iterate averaging can draw
samples from the posterior, and compare this method to other approaches. 
We assume that the cost of analyzing a minibatch is proportional to
$S$. We have shown above that we need to average over $T=N/S$ samples of SGD to create a
sample of iterate averaged SGD. Since this averaging induces a strong autocorrelation, 
we can only use every $T$th sample of the chain of averaged iterates.
Furthermore, every sample of SGD incurs a cost at least proportional to $D$ where $D$ is the dimensionality of $\theta$. 
This means that we cannot generate an independent sample from the posterior in
less than $O(S*T*D) = O(ND)$ time; we must analyze at least $N$ observations
per posterior sample.

We compare this result with the more classical
strategy of estimating the posterior mode via Newton's method (which
has an $O(ND^2 + D^3)$ cost) and then estimating the posterior
covariance by computing the inverse-Hessian at the mode, again
incurring an $O(ND^2 + D^3)$ cost. By contrast, getting an unbiased
full-rank estimate of the covariance using MCMC requires generating at
least $D$ samples, which again costs $O(ND^2)$. If $N > D$, then this
is within a constant cost of the classical approach.

However, it is conceivable that Polyak averaging (Section~\ref{sec:iterate_bayes})
could be used to estimate the first few principal components of the
posterior relatively quickly (i.e., in $O(ND)$ time). This corresponds
to finding the smallest principal components of $A$, which cannot be
done efficiently in general. A related question is investigated
experimentally in Section~\ref{sec:iter_exp}.

The analysis above implies an upper bound on the efficiency of
stochastic-gradient MCMC (SGMCMC) methods. The argument is this:
given that assumptions 1--5 hold, suppose that there exists an SGMCMC
method that, for large $N$, is able to generate effectively
independent posterior samples using $d<O(N)$ operations. Then, if we
wanted to estimate the posterior mode, we could simply average some
large number $M$ of those samples to obtain an estimator whose
covariance would be $\frac{1}{M}A^{-1}$. This approach would require
$dM$ operations, whereas iterate averaging would require $O(MN)>dM$
operations to obtain an estimator with the same covariance.  But this
contradicts the result of \citet{polyak1992acceleration} that
\emph{no} stochastic-gradient-oracle algorithm can outperform iterate
averaging. Thus, the optimality of iterate averaging as an
optimization algorithm, taken with assumptions 1--5, implies that
\emph{no SGMCMC algorithm can generate posterior samples in sublinear
  time}\footnote{At least, not without exploiting additional knowledge about
the source of gradient noise as do methods like SVRG and SAGA
\citep{johnson2013accelerating, defazio2014saga}.}.

This argument relies on assumptions 1--5 being true, but one can
easily construct scenarios in which they are violated. However, these
assumptions are all optimistic; there seems (to us) little reason to
think that problems that violate assumptions 1--5 will be easier than
those that do not.

\begin{table}
  \centering
  \begin{tabular}{|l||c|c|c|}
    \hline
    Method                   & Wine    & Skin  & Protein\\ \hline\hline
    constant SGD       & 18.7                & 0.471   & 1000.9 \\ \hline
    constant SGD-d   & 14.0                & 0.921   & 678.4 \\ \hline
    constant SGD-f    &  0.7                 & 0.005   & 1.8 \\ \hline
    SGLD~{\small \citep{welling2011bayesian}  }                  &  2.9                 &0.905   & 4.5 \\ \hline
    SGFS-d~{\small\citep{ahn2012bayesian}} &  12.8                 &0.864  & 597.4  \\ \hline
    SGFS-f~{\small\citep{ahn2012bayesian}}    &  0.8                 &0.005   & 1.3  \\ \hline
    BBVI{\small~\citep{kucukelbir2015automatic}}  &  44.7                & 5.74       & 478.1\\ \hline
  \end{tabular}
  \caption{ KL divergences between the posterior and stationary sampling
    distributions applied to the data sets discussed in Section~\ref{sec:toy_data}.
    To estimate the KL divergence, we fitted a multivariate Gaussian
    to the iterates of our sampling algorithms and used a Laplace approximation for the posterior. 
    We compared constant SGD without preconditioning and with diagonal (-d) and full rank (-f) preconditioning
    against Stochastic Gradient Langevin Dynamics and Stochastic Gradient Fisher Scoring (SGFS)
    with diagonal (-d) and full rank (-f) preconditioning, and BBVI.
  }
  \label{table:kldivergences}
\end{table}

\section{Experiments}

We test our theoretical assumptions from Section~\ref{sec:formalism}
and find good experimental evidence that they are reasonable in some
settings.  We also investigate iterate averaging and show that the
assumptions outlined in~\ref{sec:iterate_bayes} result in samples from
a close approximation to the posterior.  We also compare against other
approximate inference algorithms, including
SGLD~\citep{welling2011bayesian}, NUTS~\citep{hoffman2014no}, and black-box
variational inference (BBVI) using Gaussian reparametrization
gradients~\citep{kucukelbir2015automatic}. In
Section~\ref{sec:hyper_exp} we show that constant SGD lets us optimize
hyperparameters in a Bayesian model.

\label{sec:experiments}

\begin{figure} \centering
\includegraphics[width=.9\linewidth]{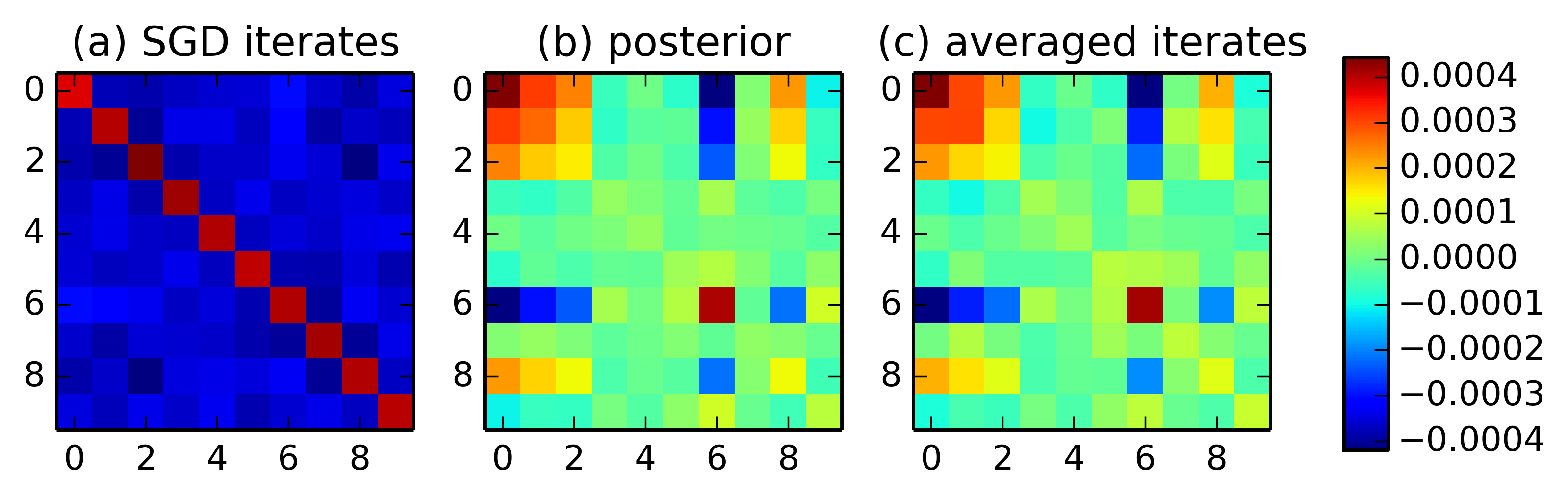}
\caption{Iterate averaging on linear regression, where we generated
artificial data as generated from the model.  (a) shows the empirical
covariance of the iterates of SGD, whereas (c) shows the averaged
iterates with optimally chosen time window. The resulting covariance
strongly resembles the true posterior covariance in (b). This shows
that iterate averaging may result in posterior sampling. }
\label{fig:iterate_averaging_panel}
\end{figure}

\subsection{Confirming the Stationary Distribution's Covariance}
\label{sec:toy_data}
In this section, we confirm empirically that the stationary
distributions of SGD with KL-optimal constant learning rates are as
predicted by the Ornstein-Uhlenbeck process.

\parhead{Real-world data.} We first considered the following data sets.  
\begin{itemize} 
\item The \emph{Wine Quality Data Set}\footnote{P. Cortez, A. Cerdeira,
F. Almeida, T. Matos and J. Reis, 'Wine Quality Data Set', UCI Machine
Learning Repository.}, containing $N=4,898$ instances, $11$ features,
and one integer output variable (the wine rating).
\item  A data set of
\emph{Protein Tertiary Structure}\footnote{Prashant Singh Rana,
'Protein Tertiary Structure Data Set', UCI Machine Learning
Repository.}, containing $N=45,730$ instances, $8$ features and one
output variable. 
\item The \emph{Skin Segmentation Data
Set}\footnote{Rajen Bhatt, Abhinav Dhall, 'Skin Segmentation Dataset',
UCI Machine Learning Repository.}, containing $N=245,057$ instances,
$3$ features, and one binary output variable. 
\end{itemize}
We applied linear
regression on data sets $1$ and $2$ and applied logistic regression on
data set $3$.  We rescaled the feature to unit length and used a
mini-batch of size $S=100$, $S=100$ and $S=10000$ for the three data sets, respectively.  The
quadratic regularizer was $1$.  The constant learning rate was
adjusted according to Eq.~\ref{eq:version_1}.

Fig.~\ref{fig:regression} shows two-dimensional projections of samples
from the posterior (blue) and the stationary distribution (cyan),
where the directions were chosen two be the smallest and largest
principal component of the posterior.  Both distributions are
approximately Gaussian and centered around the maximum of the
posterior.  To check our theoretical assumptions, we compared the
covariance of the sampling distribution (yellow) against its predicted
value based on the Ornstein-Uhlenbeck process (red), where very good
agreement was found. Since the predicted covariance is based on
approximating SGD as a multivariate Ornstein-Uhlenbeck process, we
conclude that our modeling assumptions are satisfied to a very good
extent. Since the wine dataset is smaller than the
  skin segmentation data set, it has a broader posterior and therefore
  requires a larger learning rate to match it. For this reason, discretization effects play a bigger role and
  the stationary distribution of preconditioned SGD on wine does not exactly match the posterior.
The unprojected $11$-dimensional covariances on wine data are
also compared in Fig.~\ref{fig:covariance}.  The rightmost column of
Fig.~\ref{fig:regression} shows the sampling distributions of
black box variational inference (BBVI) using the reparametrization
trick~\citep{kucukelbir2015automatic}. Our results show that the
approximation to the posterior given by constant SGD is not worse than
the approximation given by BBVI.

We also computed KL divergences between the posterior and stationary
distributions of various algorithms: constant SGD with KL-optimal
learning rates and preconditioners, Stochastic Gradient Langevin
Dynamics, Stochastic Gradient Fisher Scoring (with and without
diagonal approximation) and BBVI. For SG Fisher Scoring, we set the
learning rate to $\epsilon^*$ of Eq.~\ref{eq:version_1}, while for
Langevin dynamics we chose the largest rate that yielded stable
results ($\epsilon = \{10^{-3}, 10^{-6}, 10^{-5}\}$ for data sets $1$,
$2$ and $3$, respectively). 
Table \ref{table:kldivergences} summarizes the results.
We found that constant SGD can compete in
approximating the posterior with the MCMC algorithms under
consideration.  This suggests that the most important factor is not
the artificial noise involved in scalable MCMC, but rather the
approximation of the preconditioning matrix.

\begin{figure}[h] \centering
\includegraphics[width=\linewidth]{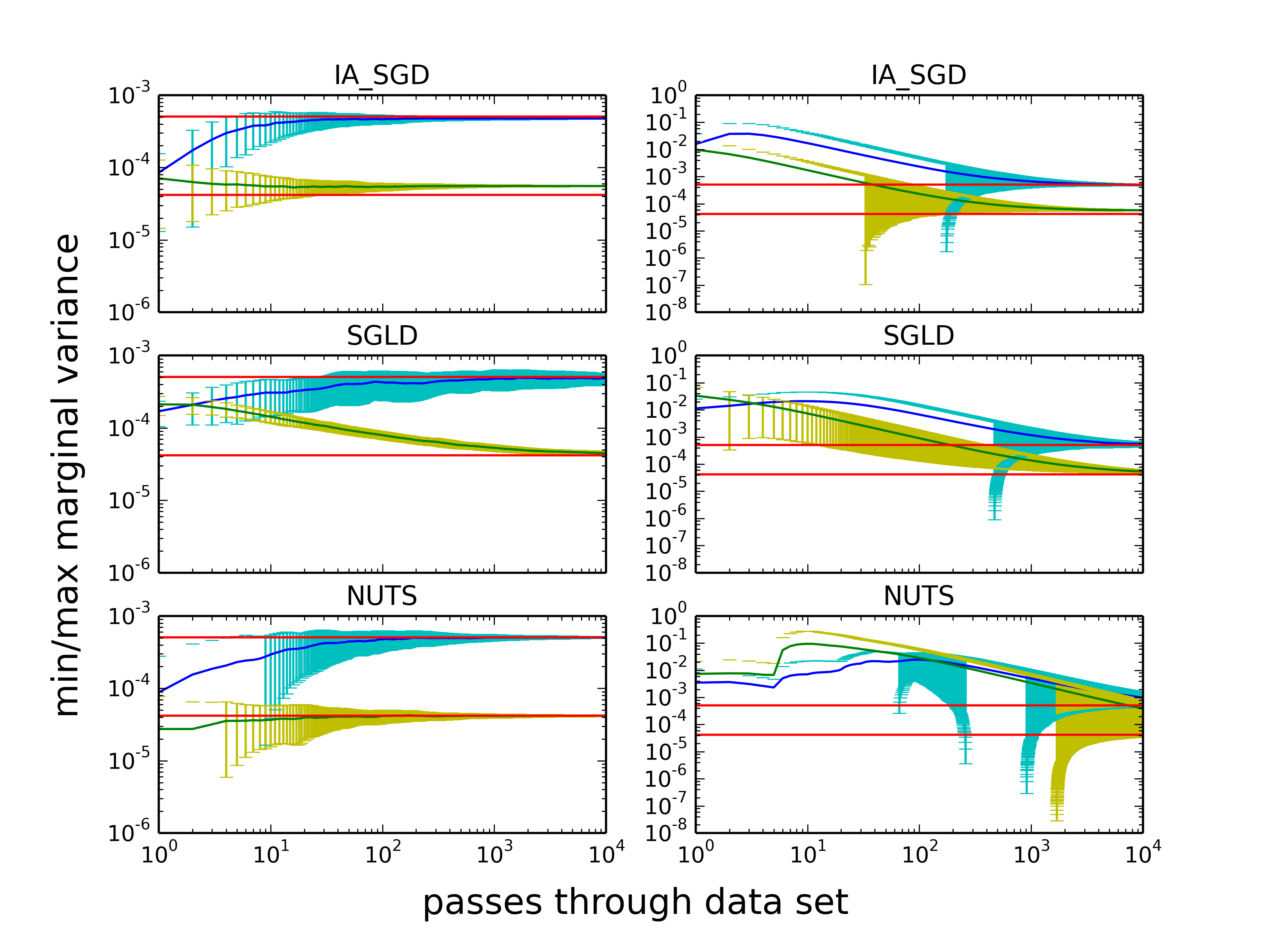}
\caption{Convergence speed comparison between IASG (top), SGLD
(middle), and NUTS (bottom) on linear regression. The plots show
minimal (yellow) and maximal (blue) posterior marginal variances,
respectively, as a function of iterations, measured in units of passes
through the data.  Error bars denote one standard deviation. Red solid
lines show the ground truth. Left plots were initialized in the
posterior maximum, whereas in the right column, we initialized
randomly.}
\label{fig:comparison_ia_sgld}
\end{figure}

\begin{figure}[h] \centering
\includegraphics[width=.7\linewidth]{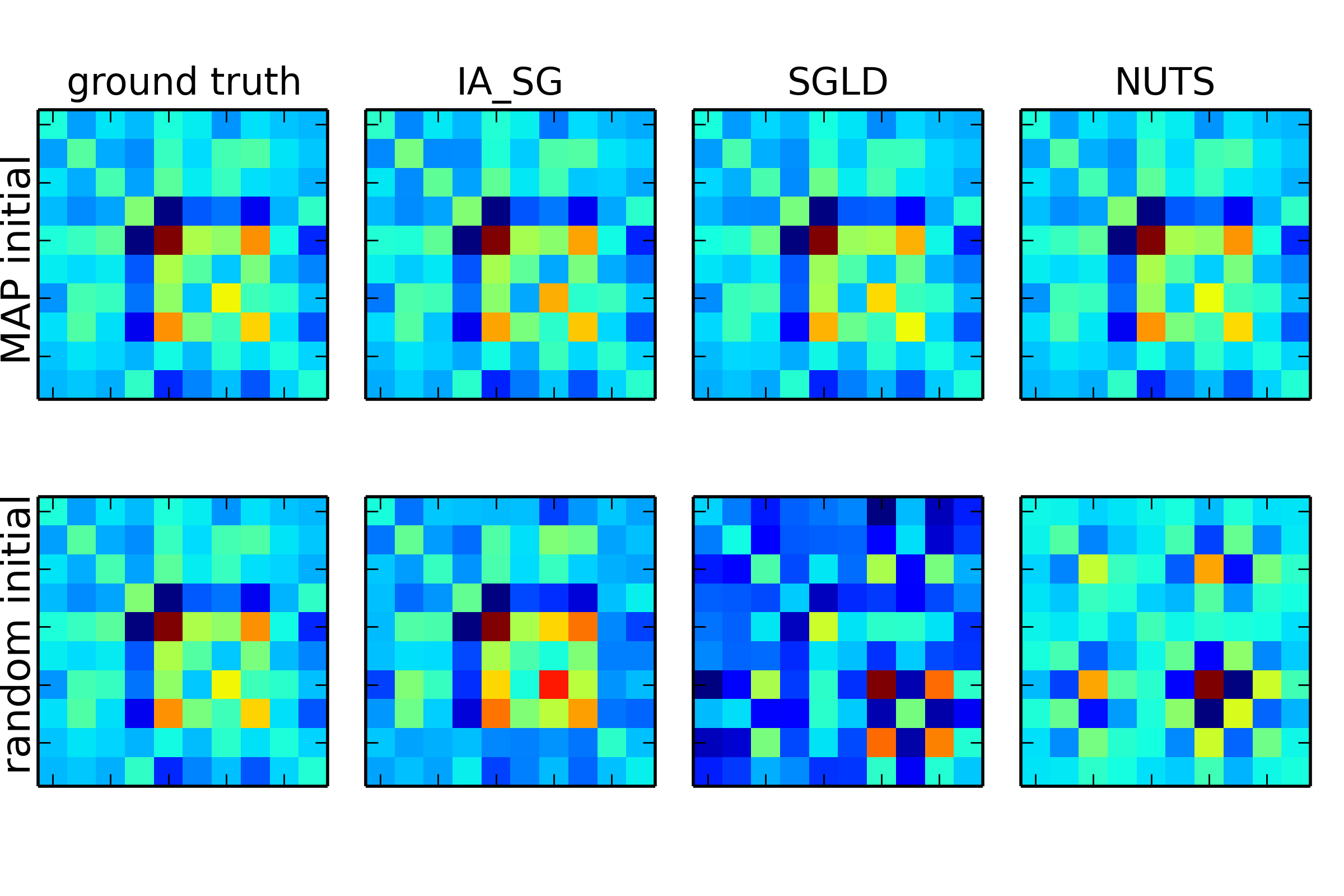}
\caption{Posterior covariances as estimated by different methods, see
also Fig.~\ref{fig:comparison_ia_sgld}. The top row shows results
where we initialized the samplers in the maximum posterior mode. The
bottom rows were initialized randomly.  For MAP
initialization, all samplers find a good estimate of the
posterior. When initializing randomly, IASG performs better than NUTS
and SGLD.}
\label{fig:comparison_ia_sgld_2}
\end{figure}

\subsection{Iterate Averaging as Approximate MCMC}
\label{sec:iter_exp}

In the following, we show that under the assumptions specified in
Section~\ref{sec:iterate_bayes}, iterate averaging with a constant
learning rate and fixed averaging window results in samples from the
posterior.

\parhead{Synthetic data.} In order to strictly satisfy the assumptions outlined in
Section~\ref{sec:iterate_bayes}, we generated artificial
data that came from the model. We chose a linear regression
model with a Gaussian prior with precision $\lambda=1$. We first
generated $N=10,000$ covariates by drawing them from a $D=10$
dimensional Gaussian with unit covariance. We drew the true
weight vector from the prior. We then generated the corresponding
response variables from the linear regression model.

In a first experiment, we confirmed that iterate averaging may result
in samples from the exact posterior, shown in Fig.~\ref{fig:iterate_averaging_panel}.
The left panel shows the empirical covariance matrix of the iterates of
SGD. The middle and right panel show the posterior covariance and
empirical covariance of the averaged iterates, respectively. As
predicted by our theory, there is a very strong resemblance which
demonstrates that constant-rate iterate averaging with the right rates
and averaging windows may actually result in samples from the
posterior. To generate this plot, we then ran constant SGD with a
constant learning rate $\epsilon=0.005$ for $10*D*N = 10^{6}$
iterations and a minibatch size $S=1$ and used an averaging window of
$N/S=10^4$, as predicted by the theory presented earlier in order to
achieve samples from the posterior.

Next, we analyzed the convergence speed of Iterate-Averaged
Stochastic Gradients (IASG, see Algorithm~\ref{alg:IASG}) compared to related
methods. Fig.~\ref{fig:comparison_ia_sgld} shows
these comparisons for three modern scalable MCMC algorithms: 
SGLD~\citep{welling2011bayesian} and NUTS~\citep{hoffman2014no}.
We investigated how quickly these algorithms could give us estimates of the posterior
covariances.  To better visualize the convergence behavior, we
focussed on diagonal entries of the posterior covariance, the 
marginal variances. As before, the data were generated
from the model such the theoretical assumptions
of~\ref{sec:iterate_bayes} applied. We then ran the samplers for up to
$10^3*D$ effective passes through the data. For IASG and SGLD we used
a minibatch size of $S=10$ and an averaging window of $N/S=1000$.
The constant learning rate of IASG  was $\epsilon=0.003$ and for SGLD we
decreased the learning rate according to the Robbins-Monro schedule of
$\epsilon_t = \textstyle{\frac{\epsilon_0}{\sqrt{1000+t}}}$ where we
found $\epsilon_0=10^{-3}$ to be optimal. NUTS automatically adjusts
its learning rate and uses non-stochastic gradients. 

The left column of Fig.~\ref{fig:comparison_ia_sgld} shows the
results of these experiments. We found the convergence speeds of the
samplers to be highly dependent on whether we optimized the samplers
in the maximum posterior mode (termed MAP-initialization: this partially eleminates the initial
bias and burn-in phase) or whether the samplers were initialized
randomly, as in a real-world application. Thus, we show both results:
while the the right column shows random initializations, the left one
shows MAP-initialization. The rows of
Fig.~\ref{fig:comparison_ia_sgld} show results of IASG (top), SGLD
(middle), and NUTS (bottom). Each entry shows the smallest and largest
marginal variance of the posterior over iterations, as estimated from
these methods, where we excluded the first 10 iterations to avoid large biases. 
We also give the standard deviations for these
estimates based on 100 independent Markov chains. The solid red lines show the ground truth of these marginal
variances. Fig.~\ref{fig:comparison_ia_sgld_2} shows additional results on the
same experiments, where we display the posterior estimates of the three
different samplers under the two different initializations. 

We found that in both initializations,
IASG can find a fast approximate solution to the
posterior. It uses stochastic gradients which gives it a competitive
advantage over NUTS (which uses full gradients) in particular in the
early search phase of the sampler. Langevin dynamics behaves similarly
at early iterations (it also employs stochastic gradients). 
However, compared to IASG, we see that the
Langevin algorithm has a much larger standard error when estmating the
posterior covariances. This is because it uses a decreasing
Robbins-Monro learning rate that slows down equilibration at long
times. In contrast, IASG uses a constant learning rate and therefore
converges fast. Note that in practice a large variance may be as bad as a large
bias, especially if posterior estimates are based on a single Markov
chain. This is also evident in Fig.~\ref{fig:comparison_ia_sgld_2},
which shows that for random initialization, both NUTS and SGLD
reveal less of the structure of the true posterior covariance compared
to IASG.

When initialized in the posterior maximum, we see that all algorithms
perform reasonably well (though SGLD's estimates are still highly variable
even after $10,000$ sweeps through the dataset).
IASG converges to a stable estimate much faster than SGLD or NUTS,
but produces slightly biased estimates of the smallest variance.

\subsection{Optimizing Hyperparameters}
\label{sec:hyper_exp}

We test the hypothesis of Section~\ref{sec:hyperparameter_theory},
namely that constant SGD as a variational algorithm gives rise to a
variational EM algorithm where we jointly optimize hyperparameters
using gradients while drawing samples form an approximate posterior.
To this end, we experimented with a
Bayesian multinomial logistic (a.k.a. softmax) regression model with
normal priors. The negative log-joint is
\begin{equation}
\begin{split}
\label{eq:logisticregression}
\mathcal{L}&\textstyle\equiv -\log p(y,\theta|x) = \frac{\lambda}{2}\sum_{d,k}\theta_{dk}^2 
- \frac{DK}{2}\log(\lambda) + \frac{DK}{2}\log2\pi
\\ &\textstyle + \sum_n \log\sum_k \exp\{\sum_d x_{nd}\theta_{dk}\}
- \sum_d x_{nd}\theta_{dy_n},
\end{split}
\end{equation}
where $n\in\{1,\ldots,N\}$ indexes examples, $d\in\{1,\ldots,D\}$
indexes features and $k\in\{1,\ldots,K\}$ indexes
classes. $x_n\in\mathbb{R}^D$ is the feature vector for the $n$th
example and $y_n\in\{1,\ldots,K\}$ is the class for that example.
Eq. \ref{eq:logisticregression} has the degenerate maximizer
$\lambda=\infty$, $\theta=0$, which has infinite posterior density which
we hope to avoid in our approach.

\parhead{Real-world data.} In all experiments, we applied this model to the MNIST dataset
($60,000$ training examples, $10,000$ test examples, $784$
features) and the cover type dataset 
($500,000$
training examples, $81,012$ testing examples, $54$ features).

Fig. \ref{fig:optimizedlambda} shows the validation loss achieved by
maximizing equation \ref{eq:logisticregression} over $\theta$ for
various values of $\lambda$. This gives rise to the continuous blue
curve. The value for constant SGD was obtained
using Eq.~\ref{eq:variationalEM}, hence using constant learning rates for
$\theta$ and decreasing learning rates for
$\lambda$. BBVI was carried out by
optimizing a variational lower bound in mean-field variational
inference,
and optimizing hyperparameters based on this lower bound.

The results suggest that BBVI and constant SGD yield similar results. 
Thus, constant SGD can be used as an inexpensive alternative to cross-validation or other
VEM methods for hyperparameter selection.

\begin{figure}
\centering
\includegraphics[width=0.5\linewidth]{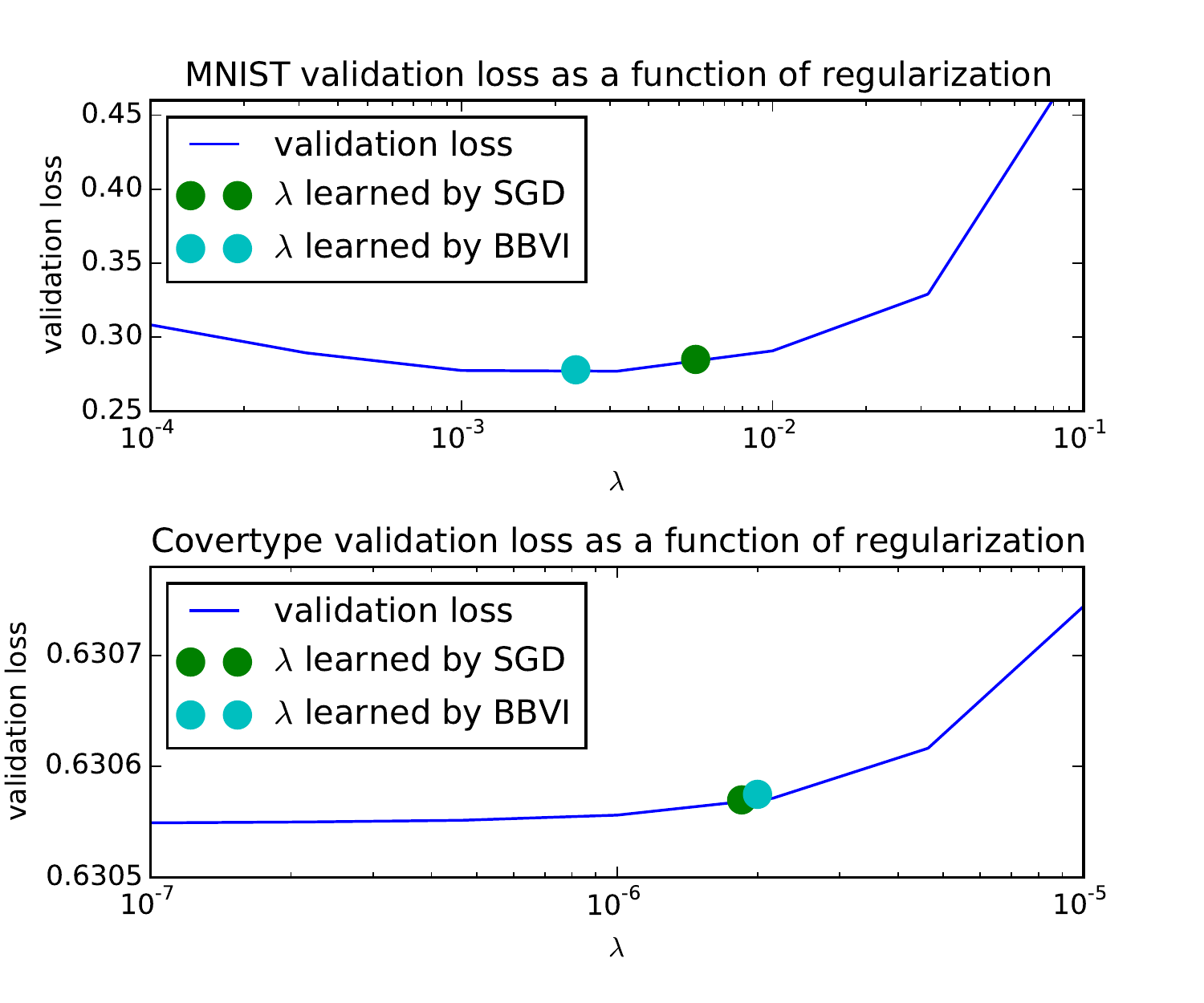}
\vskip-0.2in
\caption{Validation loss as a function of L2 regularization parameter
$\lambda$. Circles show the values of $\lambda$ that were automatically
selected by SGD and BBVI.}
\label{fig:optimizedlambda}
\end{figure}

\section{Conclusions}

In this paper, we built on a stochastic process perspective of
stochastic gradient descent and various extensions to derive several
new results. Under specified assumptions, SGD is
approximated by a multivariate Ornstein-Uhlenbeck process, which
possesses an analytic solution. We computed the
stationary distribution of constant SGD and analyzed its
properties.

We analyzed SGD together with several extensions, such as momentum,
preconditioning, and iterate averaging. The shape of the stationary distribution
is controlled by the parameters of the algorithm such as the constant learning
rate, preconditioning matrix, or averaging period. We can thus tune these parameters
to minimize the Kullback-Leibler divergence between the stationary distribution
and a Bayesian posterior. This view uses these stochastic optimization algorithms
as approximate inference. We also analyzed stochastic
gradient Langevin dynamics and stochastic gradient Fisher
scoring and were able to analyze approximation errors for these algorithms. 

The Bayesian view on constant-rate SGD allows us to use this algorithm
as a new variational EM algorithm. We suggested and tested a double SGD
scheme which uses decreasing learning rates on the hyperparameters and
constant learning rates on the main parameters. We showed that this is
both easy to implement and prevents us from finding degenerate
solutions; it is a cheap alternative to cross-validation for many
complex models.

Last, our analysis suggests the many similarities between sampling and
optimization algorithms that can be explored using the stochastic
process perspective. A future direction might be to explore
similarities in the noise characteristics of black-box variational
inference algorithms and Langevin-type MCMC. Further exploring the use
of iterate averaging as a Bayesian algorithm is another interesting
avenue for further studies.

\acks{We would like to thank Yingzhen Li and Thomas Hofmann for their valuable feedback on our manuscript.}

\appendix

\section{Examples: Ornstein-Uhlenbeck Formalism}

Let us illustrate the  
Ornstein-Uhlenbeck formalism based on two simple examples. First, 
consider the following quadratic loss,
\be
\textstyle{{\cal L}(\theta)  =  -\textstyle\frac{1}{2N} \sum_{n=1}^N || x_n - \theta||^2.}
\ee
Let us define $\bar{x} = \textstyle\frac{1}{N}\sum_{n=1}^N x_n$ as the empirical mean of the data points.
The gradient is $g (\theta) = (\bar{x}-\theta)$, and the stochastic gradient is  $\hat{g}(\theta) =   \textstyle\frac{1}{S}\sum_{s=1}^S (x_s - \theta)$.
Because the gradient is linear in $x$, the noise covariance is just the covariance of the data:
$\Sigma_x \equiv C/S  =  \textstyle\frac{1}{S^2}{\mathbb E}\left[\textstyle \sum_{s,s'}(x_s - \bar{x})(x_{s'} - \bar{x})^\top \right]  =   \textstyle\frac{1}{S} {\mathbb E}[(x_n - \bar{x})(x_n - \bar{x})^\top ].$
We can shift the parameter $\theta  \rightarrow \theta + \bar{x}$, resulting in $\theta^* = 0$. Note that the Hessian $A \equiv {\bf I}$ is just unity.
According to Eq.~\ref{eq:OU-solu},
\be
\textstyle{q(\theta)     \; \propto \; \exp\left\{ -\textstyle\frac{S}{2\epsilon} (\theta-\bar{x})^\top \Sigma_x^{-1} (\theta-\bar{x}) \right\}. }\nonumber
\ee
as the resulting stationary distribution.
Next, consider linear regression, where we minimize
\be
 \textstyle{ {\cal L}(\theta) \; = \; -\textstyle\frac{1}{2N} \sum_n (y_n - x_n^\top \theta)^2.}
\ee
We can write the stochastic gradient as $\hat{g} = \hat{A}\theta - \hat{\mu}$, where 
$\hat{\mu} = \textstyle\frac{1}{S}\sum_s x_s y_s$ and $\hat{A} = \textstyle\frac{1}{S} \sum_s x_s x_s^\top$ are estimates based on a mini-batch of size $S$. 
The sampling noise covariance is $C(\theta) = \E[(\hat{g}-g)(\hat{g}-g)^\top] = \E[\hat{g}\hat{g}^\top] - gg^\top$, where
$ \E[\hat{g}\hat{g}^\top] = \E[(\hat{A}\theta - \hat{\mu})(\hat{A}\theta - \hat{\mu})^\top] $. We see that the noise covariance is quadratic, but unfortunately
it cannot be further simplified. 

Fig.~\ref{fig:regression} shows the objective function of linear regression (blue) and the
sampling distribution of stochastic gradient descent (yellow) on simulated data. We see that
both distributions do not coincide, because the sampling distribution is also affected by the
noise covariance.

\section{Stationary Covariance}
\label{sec:covariance}

The Ornstein-Uhlenbeck process has an analytic solution in terms of the stochastic integral (Gardiner et al., 1985), 
 \be
\label{eq:OU-solu}
\theta(t) = \exp(-At) \theta(0) + \sqrt{\textstyle\frac{\epsilon}{S}}\int_0^t \exp[-A(t-t')] B dW(t')
 \ee

Following Gardiner's book and using $A = A^\top$, we derive an algebraic relation
for the stationary covariance of the multivariate Ornstein-Uhlenbeck process.
Define $\Sigma = \E[\theta(t)\theta(t)^\top]$. Using the formal solution for $\theta(t)$ given in the main paper, we find
\begin{align}
A \Sigma + \Sigma A  &={\textstyle \frac{\epsilon}{S} } \int_{-\infty}^t A \exp[-A(t-t')] BB^\top  \exp[-A(t-t')] dt' \n
					 &+{\textstyle \frac{\epsilon}{S} } \int_{-\infty}^t \exp[-A(t-t')] BB^\top  \exp[-A(t-t')] dt' A \n
					&= {\textstyle \frac{\epsilon}{S} } \int_{-\infty}^t  \frac{d}{dt'} \left(  \exp[-A(t-t')] BB^\top  \exp[-A(t-t')]  \right) \n
					&={\textstyle \frac{\epsilon}{S} }  BB^\top. \nonumber
\end{align}
We used that the lower limit of the integral vanishes by the  positivity of the eigenvalues of $A$.

\section{Reverse KL Divergence Setup}
\label{sec:app_reverseKL}

  It is interesting to also consider the case of trying to minimize the reverse KL divergence, i.e. $KL(f||q)$ instead of $KL(q||f)$.
  One might assume that this is possible since both the variational distribution and the posterior are assumed to be Gaussian.
  This turns out to lead only to a feasible algorithm in the special case where the Hessian in the optimum $A$ and the stationary covariance $\Sigma$ commute.
  In more detail, the $KL$-divergence between the posterior and the stationary distribution is (up to constants):
  \begin{align}
    KL(f||q) & = \E_f[\log f] -\E_f \left[\log q\right] \\
    & = \half \E_f \left[\theta^\top \Sigma^{-1} \theta \right] + \half \log |\Sigma| + \mathrm{const.} \n
    & = \textstyle{\frac{1}{2N}}{\rm Tr} (A^{-1} \Sigma^{-1})+ \half \log |\Sigma| + \mathrm{const.} \nonumber
  \end{align}
  While we were able to derive this divergence, it turns out that we cannot in general eliminate its dependence in the stationary
  covariance and re-express it in terms of $BB^\top$, using Eq.~\ref{eq:stationaryvariance}. However, if $A$ and $\Sigma$ commute, we can proceed as follows:
  \begin{align}
    {\rm Tr}(A^{-1} \Sigma^{-1}) & = {\rm Tr}( (\Sigma A)^{-1}) \\
    & \stackrel{A\Sigma = \Sigma A}{=} 2 {\rm Tr}( (\Sigma A + A \Sigma)^{-1}) \n
    & \stackrel{Eq.~\ref{eq:stationaryvariance}}{=} \textstyle{\frac{2S}{\epsilon}} {\rm Tr}((BB^\top)^{-1}) \nonumber
  \end{align}
  Following the logic of Theorems~\ref{thm:one} and \ref{thm:two}, we find the following result for the optimal learning rate:
  \begin{align}
    \epsilon^{*} = \frac{2S}{N D} {\rm Tr} ((BB^\top)^{-1}).\label{eq:reverseKL}
  \end{align}
Interestingly, when comparing Eq.~\ref{eq:version_1} with Eq.~\ref{eq:reverseKL}, we find that the inverse of the trace of the noise
covariance gets replaced by the trace of the inverse noise covariance. While $KL(q||f)$ suggests to choose the learning rate
inversely proportional to the largest Eigenvalue of the noise, $KL(f||q)$ thus suggests to choose the learning rate proportional to
the inverse of the smallest Eigenvalue of the noise. Both approaches have thus a different emphasis on how to fit the posterior mode, in
a similar fashion as variational inference and expectation propagation. Note, however, that $A$ and $\Sigma$ rarely commute in practice,
and thus $KL(q||f)$ is the only viable option.

\section{SGD With Momentum}
\label{sec:app_momentum}

Here, we give more details on the deviations of the results on SGD with momentum.
In order to compute the stationary distribution of Eq.~\ref{eq:momentum_continuous}, we need to solve the equations for the second moments:
\be
d\E[\theta \theta^\top] & = & \E[d\theta \theta^\top + \theta d\theta^\top] \n
				& = & (\E[\v\theta^\top] + \E[\theta \v^\top]) dt, \label{eq:moments_1} \\
d\E[\theta \v^\top] & = &  \E[d\theta \v^\top + \theta d\v^\top] \n
			& = & \E[\v\v^\top]dt  - \mu \E[\theta\v^\top]dt - \epsilon \E[\theta \theta^\top] A dt, \label{eq:moments_2}\\
d\E[ \v \theta^\top] & = &  \E[d\v \theta^\top + \v d\theta^\top] \n
			& = & \E[\v\v^\top]dt  - \mu \E[\v \theta^\top]dt - \epsilon A \E[\theta \theta^\top] dt, \label{eq:moments_3}\\
d\E[\v \v^\top] & = &  \E[d\v \v^\top + \v d\v^\top]  + \E[d\v d\v^\top] \n
		 & = & -2 \mu \E[\v\v^\top]dt - \epsilon A \E[\theta \v^\top] dt - \epsilon \E[\v \theta^\top] A dt + {\textstyle \frac{\epsilon^2}{S}}BB^\top dt.
\label{eq:moments_4}
\ee
In the last equation we used the fact that according to Ito's rule, there is an additional non-vanishing contribution due to the noise,
$ \E[d\v d\v^\top] = {\textstyle \frac{\epsilon^2}{S}} \E[B\, dW dW^\top \,B^\top ] = {\textstyle \frac{\epsilon^2}{S}} BB^\top dt$. This contribution does
not exist for the other correlators; for more details see e.g.~\citep{gardiner1985handbook}.

Since we are  looking for a stationary solution, we set the left hand sides of all equations to zero. 
Eq.~\ref{eq:moments_1} implies that $\E[\v\theta^\top] + \E[\theta \v^\top] = 0$, hence the cross-correlation between momentum and position
is anti-symmetric in the stationary state. We can thus add Eqs.~\ref{eq:moments_2}~and~\ref{eq:moments_3} to find 
$0 = d\E[\v\theta^\top + \theta\v^\top] =  2 \E[\v\v^\top]dt  - \epsilon A \E[\theta \theta^\top] dt - \epsilon \E[\theta \theta^\top] A dt$.
Combining this with Eq.~\ref{eq:moments_4} yields
\begin{align*}
\E[\v\v^\top] & =   \textstyle{\frac{\epsilon}{2}} \E[\theta \theta^\top] A +  \textstyle{\frac{\epsilon}{2}}  A \E[\theta \theta^\top] \\
 \mu \E[\v\v^\top] & =  \textstyle{\frac{\epsilon^2}{2S}} BB^\top - \underbrace{\half \epsilon (A \E[\theta \v^\top] + \E[ \v \theta^\top] A)}_{= 0}.
\end{align*}
Last, we show that the underbraced term is zero, which gives
Eq.~\ref{eq:fdt} in the main paper. Denote $\xi =
\E[\theta \v^\top] =-\xi^\top$ which is antisymmetric due to
Eq.~\ref{eq:moments_1}. First of all, $A\xi + \xi^\top A$ is
obviously symmetric. It is simultaneously antisymmetric due to the
following calculation:  $\sum_k A_{ik} \xi_{kj} + \sum_k
\xi^\top_{ik}A_{kj}^\top =  \sum_k A_{ik} \xi_{kj} -  \sum_k
\xi_{ki}A_{jk} =  \sum_k A_{ik} \xi_{kj} -  \sum_k
A_{jk}\xi_{ki}$. This term swaps the sign as $i$ and $j$ are
interchanged. Being both symmetric and antisymmetric, it is zero. 

\section{Stochastic Gradient Fisher Scoring}
\label{sec:sgfs}
We start from the Ornstein-Uhlenbeck process with minibatch size $S=1$,
\be
d \theta(t) & = & -\epsilon HA \theta(t) dt + H \left[ \epsilon B dW(t) + \sqrt{\epsilon} E dV(t)\right] \n
&\Leftrightarrow& \n
d \theta(t) & = &  -A' \theta(t) dt + B' dW(t), \nonumber
\ee
where we define $A'\equiv \epsilon HA$ and $B'\equiv H\sqrt{\epsilon^2BB^\top + \epsilon EE^\top}$,
and use the fact that the dynamics of $C dW(t) + D dV(t)$ are equivalent to
$\sqrt{CC^\top + DD^\top}dW(t)$. Here we are using matrix square roots, so
$B'B'^\top = \epsilon H(\epsilon BB^\top +  EE^\top)H$.

As derived in the paper, the variational bound is (up to a constant)
\be
KL \stackrel{\mathrm{c}}{=} \frac{N}{2}{\rm Tr}(A \Sigma) - \half \log (|\Sigma|). \nonumber
\ee
To evaluate it, the task is to remove the unknown covariance $\Sigma$ from the bound. To this end, as before, we use the identity 
for the stationary covariance $A'\Sigma + \Sigma A'^\top = B'B'^\top$.
The criterion for the stationary covariance is equivalent to
\be
HA\Sigma + \Sigma A H &=& \epsilon H BB^\top H + HEE^\top H. \nonumber
\ee

We can further simplify this expression as follows:
\be
A\Sigma + H^{-1} \Sigma A H & = &  \epsilon BB^\top H + EE^\top H \n
\Rightarrow {\rm Tr}(A\Sigma) & = & \frac{1}{2}{\rm Tr}(H( \epsilon BB^\top + EE^\top)). \nonumber
\ee
As above, we can also reparameterize the covariance as $\Sigma = TH$, so that
$T$ does not depend on $H$:
\begin{equation}
\begin{split}
HA\Sigma + \Sigma A H &= \epsilon H BB^\top H + HEE^\top H \n
A\Sigma H^{-1} + H^{-1}\Sigma A &= \epsilon BB^\top + EE^\top \n
AT + T^\top A &= \epsilon BB^\top + EE^\top.
\end{split}
\end{equation}
The KL divergence is therefore
\be
KL & = & \frac{N}{2}{\rm Tr} (A\Sigma) - \frac{D}{2} - \frac{1}{2}\log(N|A|) - \frac{1}{2}\log|\Sigma| \n
& = & \frac{N}{4}{\rm Tr} (H( \epsilon BB^\top + EE^\top)) - \frac{D}{2} - \frac{1}{2}\log(N|A|) - \frac{1}{2}\log|T| - \frac{1}{2}\log|H|,
\ee
which is the result we give in the main text. 

\section{Square Root Preconditioning}
\label{sec:sqrt}

We analyze the case where we precondition with a matrix that
is proportional to the square root of the diagonal entries of the
noise covariance.

We define
\be
G &=& \sqrt{{\rm diag}(BB^\top)} \nonumber
\ee
 as the diagonal matrix that contains square roots of the diagonal elements of the noise covariance.
We use an additional scalar learning rate $\epsilon$ .

\begin{theorem}[Taking square roots] 
Consider SGD preconditioned with $G^{-1}$ as defined above. 
Under the previous assumptions,  the constant learning rate which minimizes
KL divergence between the stationary distribution of this process and the posterior is
\be
\label{eq:version_3}
\epsilon^* & = &\textstyle  \frac{2DS}{N{\rm Tr}(BB^\top G^{-1})}.
\ee
\end{theorem}

\begin{proof}
We read off the appropriate KL divergence from the proof of Theorem~\ref{thm:two} with $G^{-1}\equiv H$:
\begin{equation*}
KL(q  || f )   \stackrel{\mathrm{c}}{=}   \textstyle\frac{\epsilon N}{2S}{\rm Tr}(BB^\top G^{-1}) - {\rm Tr}\log(G) +\textstyle\frac{D}{2} \log \textstyle\frac{\epsilon}{S} - \frac{1}{2}\log |\Sigma|
\end{equation*}
Minimizing this KL divergence over the learning rate $\epsilon$ yields Eq.~\ref{eq:version_3}.
\end{proof}

\
\section{Iterage Averaging}
\label{sec:iterate_proof}
We now prove our result for the covariance of the averaged
iterates.
We first need an identity for the non-equal-time covariance in the stationary state:
\be
{\mathbb E}[\theta(t)\theta(s)^\top] & = & \begin{cases} \Sigma
  e^{-\epsilon A
    (s-t)} & t < s \\ e^{-\epsilon A (t-s)}  \Sigma & t \geq s. \end{cases}
\label{eq:app_autocorr}
\ee
To derive it, one uses the formal solution of the Ornstein-Uhlenbeck process for $\theta(t)$
in combination with Eq.~\ref{eq:estimatorcovariance}, see
also~\citep{gardiner1985handbook} for more details.  Note that for $t=s$,
it simplifies to $\mathbb{E}[\theta(t)\theta(t)^\top] = \Sigma$, as one would expect.

We are averaging over $T$ time steps. Going to the continuous-time OU formalism,
we are interested in the following quantity,
which is the equal-time covariance of the time-averaged iterates:
\begin{equation*}
  \begin{split}
    D &\equiv \E\left[
      \left(
        \frac{1}{T}\int_0^{T}
        \theta(t)dt
      \right)
      \left(
        \frac{1}{T}\int_0^{T}
        \theta(s)ds
      \right)^\top
    \right].
  \end{split}
\end{equation*}
This can be further broken down to two contributions:
\begin{equation*}
  \begin{split}
    D &=
    \frac{1}{T^2}\int_0^T\int_0^T \E[\theta(t)\theta(s)^\top] dsdt
    \\ &=
    \frac{1}{T^2}\int_0^T\int_0^t \E[\theta(t)\theta(s)^\top] dsdt
    + \frac{1}{T^2}\int_0^T\int_t^T \E[\theta(t)\theta(s)^\top] dsdt
\end{split}
\end{equation*}
We now use the eigendecomposition $A=U\Lambda U^\top$ of the Hessian,
the autocorrelation Eq.~\ref{eq:app_autocorr}, 
as well as the identity $e^{cA}
= Ue^{c\Lambda}U^\top$. The first term becomes
\begin{equation}
  \begin{split}
    \frac{1}{T^2}\int_0^T\int_0^t \E[\theta(t)\theta(s)^\top] dsdt
    &= \frac{1}{T^2}\int_0^T \int_0^t  e^{-\epsilon A(t-s)} \Sigma dsdt
    \\ &= \frac{1}{T^2}\int_0^T \int_0^t  U e^{-\epsilon \Lambda(t-s)} U^\top \Sigma dsdt
    \\ &= \frac{1}{\epsilon T^2}\int_0^T U \Lambda^{-1}({\bf I} -
    e^{-\epsilon t \Lambda}) U^\top \Sigma dt
    \\ &= \frac{1}{\epsilon T} A^{-1} \Sigma +
    \frac{1}{\epsilon^2 T^2} U \Lambda^{-2}(e^{-\epsilon T\Lambda} - {\bf I}) U^\top \Sigma.
\end{split}
\end{equation}
The calculation for the second term goes analogously and yields
\begin{equation*}
  \begin{split}
    \frac{1}{T^2}\int_0^T\int_t^T \E[\theta(t)\theta(s)^\top] dsdt
 &= \frac{1}{\epsilon T}  \Sigma A^{-1} +
    \frac{1}{\epsilon^2 T^2} \Sigma U \Lambda^{-2}(e^{-\epsilon T\Lambda} - {\bf I})
    U^\top.
\end{split}
\end{equation*}
Both equations combined give us
\begin{equation}
  \begin{split}
D &=    \frac{1}{\epsilon T} ( A^{-1} \Sigma+\Sigma A^{-1}) \\
    & +  \frac{1}{\epsilon^2 T^2} (U \Lambda^{-2}(e^{-\epsilon
      T\Lambda} - {\bf I}) U^\top \Sigma + \Sigma U \Lambda^{-2}(e^{-\epsilon
      T\Lambda} - {\bf I}) U^\top).
\end{split}
\label{eq:app_D_full}
\end{equation}
When $\epsilon T \Lambda \gg 1$ (valid for sufficiently long averaging
periods $T$), we obtain
$$D \approx
\frac{1}{\epsilon T} ( A^{-1} \Sigma+\Sigma A^{-1}),
$$
which is Eq.~\ref{eq:iterate_autocorr} in the main text.
We can also simplify the expression for $\Sigma = \frac{\epsilon}{2S}I$, as motivated in
Section~\ref{sec:iterate_bayes}. In this case, Eq.~\ref{eq:app_D_full} results in
\begin{equation*}
  \begin{split}
D &=    \frac{1}{S T} A^{-1}  +  \frac{1}{\epsilon S T^2} (U \Lambda^{-2}(e^{-\epsilon
      T\Lambda} - {\bf I}) U^\top).
\end{split}
\end{equation*}
This is exactly Eq.~\ref{eq:stationarycovariance_averaging} in the
main text.

\bibliography{references}

\end{document}